\documentclass[10pt,twocolumn,letterpaper]{article}

\usepackage{iccv}
\usepackage{times}
\usepackage{epsfig}
\usepackage{graphicx}
\usepackage{amsmath}
\usepackage{amssymb}
\usepackage{authblk}

\usepackage{float}
\usepackage{mathtools}
\usepackage{framed}
\usepackage{pbox}
\usepackage{afterpage}
\usepackage{url}
\usepackage{array}
\usepackage{amsfonts}
\usepackage[table]{xcolor}

\usepackage{multirow}
\usepackage{booktabs}
\usepackage{relsize}
\usepackage{xspace}
\usepackage{caption}
\usepackage{subcaption}
\newcommand{\ignore}[1]{}

\usepackage{framed}	% frame
\usepackage{url}
\usepackage{enumerate}
\usepackage{color}

\usepackage{changepage}

\usepackage{algpseudocode}
\usepackage{algorithm}

\usepackage{amsthm}
\theoremstyle{definition}

\newtheorem{pro}{Proposition}[section]

\newtheorem*{rmk}{Remark}
\theoremstyle{plain}
\newtheorem{thm}{Theorem}[section]
\newtheorem{lem}{Lemma}[section]

\newcommand{\SKIP}[1]{}
\newcommand{\calV}{\mathcal{V}}

\newcommand{\bfx}{\mathbf{x}}

\newcommand{\bfy}{\mathbf{y}}
\newcommand{\bfz}{\mathbf{z}}

\newcommand{\I}{\mathbbm{1}}

\newcommand{\R}{\rm I\!R}

\newcommand{\calC}{\mathcal{C}}

\newcommand{\calD}{\mathcal{D}}

\newcommand{\calW}{\mathcal{W}}
\newcommand{\bfw}{\mathbf{w}}

\newcommand{\bfu}{\mathbf{u}}

\newcommand{\calQ}{\mathcal{Q}}
\newcommand{\bfq}{\mathbf{q}}
\newcommand{\tbfw}{\tilde{\mathbf{w}}}
\newcommand{\tbfu}{\tilde{\mathbf{u}}}

\newcommand{\calS}{\Delta^{\!m}}
\newcommand{\tbfq}{\tilde{\mathbf{q}}}
\newcommand{\sign}{\operatorname{sign}}

\newcommand{\softmax}{\operatorname{softmax}}
\newcommand{\hardmax}{\operatorname{hardmax}}
\newcommand{\sparsemax}{\operatorname{sparsemax}}
\newcommand{\tL}{\tilde{L}}
\newcommand{\tu}{\tilde{u}}
\newcommand{\tw}{\tilde{w}}

\newcommand{\bfg}{\mathbf{g}}

\newcommand{\hL}{\hat{L}}

\newcommand{\allweights}{\{1,\ldots, m\}}
\newcommand{\alllabels}{\{1,\ldots, d\}}

\newcommand{\projS}{P_{\calS}}

\newcommand{\kl}[2]{\operatorname{KL(#1\|#2)}}
\newcommand{\E}{\mathbb{E}}

\def\amin#1{{\underset{#1}{\operatorname{argmin}}}}%
\def\amax#1{{\underset{#1}{\operatorname{argmax}}}}%

\newcommand{\NOTE}[1]{}
\def\myref#1{{\color{red}{#1}}}%

\usepackage{xcolor}

\usepackage{bbm}

\graphicspath{{images/}}

\usepackage[page]{appendix}

\usepackage{chngcntr}

\usepackage{etoolbox}

\AtBeginEnvironment{subappendices}{%
\section*{Appendix}
\addcontentsline{toc}{section}{Appendix}
}

\usepackage{xspace}
\makeatletter
\DeclareRobustCommand\onedot{\futurelet\@let@token\@onedot}
\def\@onedot{\ifx\@let@token.\else.\null\fi\xspace}

\def\eg{\emph{e.g}\onedot} 
\def\ie{\emph{i.e}\onedot}

\def\wrt{w.r.t\onedot} 

\makeatother

\usepackage{enumitem}
\newenvironment{tight_itemize}{
\begin{itemize}[leftmargin=10pt]
  \setlength{\topsep}{0pt}
  \setlength{\itemsep}{0pt}
  \setlength{\parskip}{0pt}
  \setlength{\parsep}{0pt}
}{\end{itemize}}
\newenvironment{tight_enumerate}{
\begin{enumerate}[leftmargin=10pt]
  \setlength{\topsep}{0pt}
  \setlength{\itemsep}{0pt}
  \setlength{\parskip}{0pt}
  \setlength{\parsep}{0pt}
}{\end{enumerate}}

\def\figref#1{Fig.~\ref{#1}}

\def\secref#1{Sec.~\ref{#1}}
\def\eqref#1{Eq.~(\ref{#1})}

\def\plaineqref#1{(\ref{#1})}
\def\algref#1{Algorithm~\ref{#1}}

\def\thmref#1{Theorem~\ref{#1}}
\def\lemref#1{Lemma~\ref{#1}}
\def\proref#1{Proposition~\ref{#1}}
\def\tabref#1{Table~\ref{#1}}

\algnewcommand\INPUT{\item[\textbf{Input:}]}%
\algnewcommand\OUTPUT{\item[\textbf{Output:}]}%

\usepackage[acronym,smallcaps]{glossaries} 
\makeglossaries 
\newacronym{MRF}{mrf}{Markov Random Field}
\newacronym{SGD}{sgd}{Stochastic Gradient Descent}
\newacronym{GD}{gd}{Gradient Descent}
\newacronym{PGD}{pgd}{Projected Gradient Descent}
\newacronym{ICM}{icm}{Iterative Conditional Modes}
\newacronym{DNN}{dnn}{Deep Neural Networks}
\newacronym{NN}{nn}{Neural Network}
\newacronym{PMF}{pmf}{Proximal Mean-Field}
\newacronym{PICM}{picm}{Proximal Iterative Conditional Modes}
\newacronym{BC}{bc}{BinaryConnect}
\newacronym{BWN}{bwn}{Binary Weight Network}
\newacronym{IP}{ip}{Integer Programming}
\newacronym{fc}{fc}{fully-connected}
\newacronym{REF}{ref}{Reference Network}
\newacronym{KL}{kl}{KL}
\newacronym{LR}{lr}{LR}
\newacronym{MAP}{map}{Maximum a Posteriori}
\newacronym{PQ}{pq}{ProxQuant}

\newcommand{\slenet}[1]{{\small LeNet#1}}

\newcommand{\svgg}[1]{{\small VGG#1}}
\newcommand{\sresnet}[1]{{\small ResNet#1}}

\newcommand{\srelu}{{\small ReLU}}
\newcommand{\mnist}{{\small MNIST}}

\newcommand{\cifar}[1]{{\small CIFAR#1}}
\newcommand{\tinyimagenet}{{TinyImageNet}}

\usepackage{pifont}
\newcommand{\fro}{}
\newif\ifsupp
\suppfalse

\newif\ifarxiv
\arxivfalse
\def\arxivcopy{\global\arxivtrue}

\usepackage[pagebackref=true,breaklinks=true,letterpaper=true,colorlinks,bookmarks=false]{hyperref}

\iccvfinalcopy % *** Uncomment this line for the final submission
\arxivcopy % *** Uncomment this line for the for arxiv copy (paper + appendix)  

\ifarxiv
\ifsupp
\global\suppfalse
\fi
\fi

\def\iccvPaperID{3682} % *** Enter the ICCV Paper ID here

\makeatletter
\renewcommand\AB@affilsepx{\quad\quad	\protect\Affilfont}
\makeatother
\ifsupp
\title{Proximal Mean-field for Neural Network Quantization\\ Supplementary Material}
\else
\title{Proximal Mean-field for Neural Network Quantization}
\fi

\author[1]{Thalaiyasingam Ajanthan\thanks{Part of the work was done while at the University of Oxford.}}
\author[2]{Puneet K. Dokania}
\author[1]{Richard Hartley}
\author[2]{Philip H. S. Torr}
\affil[1]{Australian National University}
\affil[2]{University of Oxford}

\ificcvfinal\fi
\begin{document}

\maketitle

\ifsupp
\appendix
%=====================================================================
%\clearpage
%\appendices

Here, we provide the proofs of propositions and theorems stated in the main
paper and a self-contained overview of the mean-field method.
Later in~\secref{ap:expr}, we give the experimental details to allow
reproducibility, and more empirical analysis for our \gls{PMF} algorithm. 

% END SKIP

\section{Mean-field Method}\label{ap:mf}
For completeness we briefly review the underlying theory of the mean-field
method.
For in-depth details, we refer the interested reader to the Chapter~\myref{5}
of~\cite{wainwright2008graphical}.
Furthermore, for background on \acrfull{MRF}, we refer the reader to the
Chapter~\myref{2} of~\cite{ajanthanphdthesis}.
In this section, we use the notations from the main paper and highlight the
similarities wherever possible.

\paragraph{Markov Random Field.}
Let $\calW = \{W_1,\ldots,W_m\}$ be a set of random variables, where each random
variable $W_j$ takes a label $w_j\in\calQ$. 
For a given labelling $\bfw\in\calQ^m$, the energy associated with an \gls{MRF}
can be written as:
\begin{equation}\label{eq:mrfe}
L(\bfw) = \sum_{C\in\calC} L_C(\bfw)\ , 
\end{equation}
where $\calC$ is the set of subsets (cliques) of $\calW$ and $L_C(\bfw)$ is a
positive function (factor or clique potential) that depends only on the values
$w_j$ for $j\in C$.
Now, the joint probability distribution over the random variables can be written
as:
\begin{equation}\label{eq:mrfp}
P(\bfw) = \frac{1}{Z}\exp(-L(\bfw))\ , 
\end{equation}
 where the normalization constant $Z$ is usually referred to as the partition
function.
 From Hammersley-Clifford theorem, for the factorization given
in~\eqref{eq:mrfe}, the joint probability distribution $P(\bfw)$ can be shown to
factorize over each clique $C\in\calC$, which is essentially the Markov
property.
 However, this Markov property is not necessary to write~\eqref{eq:mrfp} and in
turn for our formulation, but since mean-field is usually described in the
context of \gls{MRF}s we provide it here for completeness.
 The objective of mean-field is to obtain the most probable configuration, which
is equivalent to minimizing the energy $L(\bfw)$.

\paragraph{Mean-field Inference.} 
 The basic idea behind mean-field is to approximate the intractable probability
distribution $P(\bfw)$ with a tractable one.
 Specifically, mean-field obtains a fully-factorized distribution (\ie, each
random variable $W_j$ is independent) closest to the true distribution $P(\bfw)$
in terms of \acrshort{KL}-divergence. 
 % Specifically, mean-field approximates the true distribution $P(\bfw)$ with a
% fully-factorized distribution (\ie, each random variable $W_j$ is independent)
% by minimizing the \acrshort{KL}-divergence between them.
Let $U(\bfw) = \prod_{j=1}^m U_{j}(w_j)$ denote a fully-factorized
distribution.
%Here, the notation $U_{j}(w_j)$ denotes the probability of random variable
% $W_j$ taking the label $w_j$.
\SKIP{ 
For simplicity, we introduce variables $u_{j:\lambda}$ to denote the probability
of random variable $W_j$ taking the label $\lambda$, where the vector $\bfu$
satisfy:
\begin{equation}
\bfu\in\calS = \left\{\begin{array}{l|l}
\multirow{2}{*}{$\bfu$} & \sum_{\lambda} u_{j:\lambda} = 1, \quad\forall\,j\\
&u_{j:\lambda} \ge 0,\ \ \ \quad\quad\forall\,j, \lambda \end{array} \right\}\
.
\end{equation}
}
Recall, the variables $\bfu$ introduced in~Sec.~\myref{2.2} represent the
probability of each weight $W_j$ taking a label $q_\lambda$.
Therefore, the distribution $U$ can be represented using the variables
$\bfu\in\calS$, where $\calS$ is defined as:
\begin{equation}
\calS = \left\{\begin{array}{l|l}
\multirow{2}{*}{$\bfu$} & \sum_{\lambda} u_{j:\lambda} = 1, \quad\forall\,j\\
&u_{j:\lambda} \ge 0,\ \ \ \quad\quad\forall\,j, \lambda \end{array} \right\}\
.
\end{equation} 
The \acrshort{KL}-divergence between $U$ and $P$ can be written as:
\begin{align}\label{eq:mfkl}
\kl{U}{P} &= \sum_{\bfw\in\calQ^m} U(\bfw)\log\frac{U(\bfw)}{P(\bfw)}\
,\\\nonumber
 &= \sum_{\bfw\in\calQ^m} U(\bfw)\log U(\bfw) - \sum_{\bfw\in\calQ^m}
U(\bfw)\log P(\bfw)\ ,\\\nonumber
 &= -H(U) - \sum_{\bfw\in\calQ^m} U(\bfw)\log \frac{\exp(-L(\bfw))}{Z}\
,\quad\mbox{\eqref{eq:mrfp}}\ ,\\\nonumber
 &= -H(U) + \sum_{\bfw\in\calQ^m} U(\bfw)L(\bfw) + \log Z\ .\\\nonumber
\end{align}
Here, $H(U)$ denotes the entropy of the fully-factorized distribution. 
%, which is exactly the one used in~\thmref{thm:spgd_mf1}.
Specifically,
\begin{equation}
H(U) = H(\bfu) = -\sum_{j=1}^m\sum_{\lambda=1}^{d}u_{j:\lambda}\,\log
u_{j:\lambda}\ .
\end{equation} 
Furthermore, in~\eqref{eq:mfkl}, since $Z$ is a constant, it can be removed from
the minimization.
Hence the final mean-field objective can be written as:
\begin{align}
\min_{U} F(U) &:= \sum_{\bfw\in\calQ^m} U(\bfw)L(\bfw) - H(U)\ ,\\\nonumber
&= \E_{U}[L(\bfw)] - H(U)\ ,\\\nonumber
\end{align}
where $\E_{U}[L(\bfw)]$ denotes the expected value of the loss $L(\bfw)$ over
the distribution $U(\bfw)$.
Note that, the expected value of the loss can be written as a function of the
variables $\bfu$. 
In particular, 
\begin{align}
E(\bfu) &:= \E_{U}[L(\bfw)] =  \sum_{\bfw\in\calQ^m} U(\bfw)L(\bfw)\
,\\\nonumber
&=  \sum_{\bfw\in\calQ^m} \prod_{j=1}^m u_{j:w_j} L(\bfw)\ .\\\nonumber
\end{align}
Now, the mean-field objective can be written as an optimization over $\bfu$:
\begin{equation}\label{eq:mfobj}
\min_{\bfu\in\calS} F(\bfu) := E(\bfu) - H(\bfu)\ .
\end{equation}
Computing this expectation $E(\bfu)$ in general is intractable as the sum is
over an exponential number of elements ($|\calQ|^m$ elements, where $m$ is
usually in the order millions for an image or  a neural network).
However, for an \gls{MRF}, the energy function $L(\bfw)$ can be factorized
easily as in~\eqref{eq:mrfe} (\eg, unary and pairwise terms) and $E(\bfu)$ can
be computed fairly easily as the distribution $U$ is also fully-factorized.

In mean-field, the above objective~\plaineqref{eq:mfobj} is minimized
iteratively using a fixed point update. 
This update is derived by writing the Lagrangian and setting the derivatives
with respect to $\bfu$ to zero.
%The derivation is very similar to the proof of~\lemref{lem:sm1}, and at
At iteration $k$, the mean-field update for each $j\in\allweights$ can be written
as:
\begin{equation}
u^{k+1}_{j:\lambda} = \frac{\exp(-\partial E^k/\partial
u_{j:\lambda})}{\sum_{\mu} \exp(-\partial E^k/\partial u_{j:\mu})}\quad
\forall\,\lambda\in\alllabels\ .
\end{equation}
Here, $\partial E^{k}/\partial u_{j:\lambda}$ denotes the gradient of
$E(\bfu)$ with respect to $u_{j:\lambda}$ evaluated at $u^k_{j:\lambda}$.
This update is repeated until convergence.
Once the distribution $U$ is obtained, finding the most probable configuration
is straight forward, since $U$ is a product of independent distributions over
each random variable $W_j$.
Note that, as most probable configuration is exactly the minimum label
configuration, the mean-field method iteratively minimizes the actual energy
function $L(\bfw)$.

%End SKIP

\section{BinaryConnect as Proximal ICM}\label{ap:bc_vs_pgdh}
\begin{pro}\label{pro:bc_vs_pgdh1}
% Consider algorithms \gls{BC} and \gls{PICM}. Let $\tbfw^0 =
% \tbfu^0\bfq$, where $\bfq = [-1,1]^T$. Then, for a given $\eta_{\bfw} > 0$,
% \begin{enumerate}
%   \item if $\tbfw^k = \tbfu^k\bfq$, then\\ $\bfw^k = \sign(\tbfw^k) =
%   \hardmax(\tbfu^k)\bfq$.
%   
%   \item for \gls{PICM}, if $\eta_{\bfu} = \eta_{\bfw}/2$, then\\
%   $\tbfw^{k+1} = \tbfu^{k+1}\bfq$.
% \end{enumerate} 
Consider \gls{BC} and \gls{PICM} with $\bfq = [-1,1]^T$ and $\eta_{\bfw} > 0$.
For an iteration $k>0$, if $\tbfw^k= \tbfu^k\bfq$ then, 
\begin{enumerate}
  \item the projections in \acrshort{BC}: ${\bfw^k = \sign(\tbfw^k)}$ and\\ 
   \acrshort{PICM}: ${\bfu^k = \hardmax(\tbfu^k)}$  
  satisfy  ${\bfw^k = \bfu^k\bfq}$.
  
  \item let the learning rate of \acrshort{PICM} be  $\eta_{\bfu} = \eta_{\bfw}/2$, then the updated points after the
gradient descent step in \acrshort{BC} and \acrshort{PICM} satisfy
  ${\tbfw^{k+1} = \tbfu^{k+1}\bfq}$.
\end{enumerate}
\end{pro}
\begin{proof}
\begin{enumerate}
  \item In the binary case ($\calQ = \{-1,1\}$), for each $j\in\allweights$,
  the $\hardmax$ projection can be written as:
  \begin{align}\label{eq:puarg}
  u^k_{j:-1} &= \left\{\begin{array}{ll}
1 & \mbox{if $\tu^k_{j:-1} \ge \tu^k_{j:1}$}\\\nonumber
0 & \mbox{otherwise} \end{array} \right.\ ,\\
u^k_{j:1} &= 1- u^k_{j:-1}\ .
  \end{align}
  Now, multiplying both sides by $\bfq$, and substituting $\tw_j^k =  
  \tbfu^k_j\bfq$,
  \begin{align}
   \bfu^k_j\bfq &= \left\{\begin{array}{ll}
-1 & \mbox{if $\tw_j^k = -1\,\tu^k_{j:-1} + 1\,\tu^k_{j:1} \le 0$}\\\nonumber
1 & \mbox{otherwise} \end{array} \right.\ ,\\
w_j^k &= \sign(\tw_j^k)\ .
  \end{align}
  Hence, $\bfw^k = \sign(\tbfw^k) =  \hardmax(\tbfu^k)\bfq$.
  
  \item Since $\bfw^k = \bfu^k\bfq$ from case (1) above, by chain rule the
  gradients $\bfg_{\bfw}^k$ and $\bfg_{\bfu}^k$ satisfy,
  \begin{equation}\label{eq:guw}
  \bfg_{\bfu}^k = \bfg_{\bfw}^k\,\frac{\partial \bfw}{\partial \bfu} =
  \bfg_{\bfw}^k\,\bfq^T\ .
  \end{equation}
  Similarly, from case (1) above, for each $j\in\allweights$,
  \begin{align}
  w^k_j &= \sign(\tw^k_j) = \sign(\tbfu^k_j\,\bfq) = 
  \hardmax(\tbfu^k_j)\,\bfq\ ,\\[-0.5cm]\nonumber
  \frac{\partial w_j}{\partial \tbfu_j} &= \frac{\partial \sign}{\partial
  \tbfu_j} = \frac{\partial \sign}{\partial \tw_j}
  \frac{\partial \tw_j}{\partial \tbfu_j} = \frac{\partial \hardmax}{\partial
  \tbfu_j}\, \bfq\ .
  \end{align}
  Here, the partial derivatives are evaluated at $\tbfu=\tbfu^k$ but
  omitted for notational clarity.
  Moreover, $\frac{\partial w_j}{\partial \tbfu_j}$ is a $d$-dimensional column
  vector, $\frac{\partial \sign}{\partial \tw_j}$ is a scalar, and
  $\frac{\partial \hardmax}{\partial \tbfu_j}$ is a $d \times d$ matrix.
  Since, $\frac{\partial \tw_j}{\partial \tbfu_j} = \bfq$ (similar
  to~\eqref{eq:guw}),
  \begin{equation}\label{eq:pdwu}
  \frac{\partial w_j}{\partial \tbfu_j} = \frac{\partial \sign}{\partial
  \tw_j}\, \bfq = \frac{\partial \hardmax}{\partial \tbfu_j}\, \bfq\ .
  \end{equation}
  Now, consider the $\bfg_{\tbfu}^k$ for each $j\in\allweights$,
  \begin{align}
  \bfg_{\tbfu_j}^k &= \bfg_{\bfu_j}^k\,\frac{\partial \bfu_j}{\partial \tbfu_j}
  = \bfg_{\bfu_j}^k\,\frac{\partial \hardmax}{\partial \tbfu_j}\ ,\\\nonumber
  \bfg_{\tbfu_j}^k\,\bfq &= \bfg_{\bfu_j}^k\,\frac{\partial \hardmax}{\partial
  \tbfu_j}\,\bfq\ ,\quad\mbox{multiplying by $\bfq$}\ ,\\\nonumber
  &= g_{w_j}^k\,\bfq^T\,\frac{\partial \hardmax}{\partial
  \tbfu_j}\,\bfq\ ,\quad\mbox{\eqref{eq:guw}}\ ,\\\nonumber
  &= g_{w_j}^k\,\bfq^T\,\frac{\partial \sign}{\partial
  \tw_j}\, \bfq\ ,\quad\mbox{\eqref{eq:pdwu}}\ ,\\\nonumber
  &= g_{w_j}^k\,\frac{\partial \sign}{\partial
  \tw_j}\,\bfq^T\,\bfq\ ,\\\nonumber
  &= g_{\tw_j}^k\,\bfq^T\,\bfq\ ,\quad\mbox{$\frac{\partial \sign}{\partial
  \tw_j} = \frac{\partial w_j}{\partial
  \tw_j}$}\ ,\\\nonumber
  &= 2\,g_{\tw_j}^k\ ,\quad\mbox{$\bfq = [-1, 1]^T$}\ .
  \end{align}
  Now, consider the gradient descent step for $\tbfu$, with $\eta_{\bfu} =
  \eta_{\bfw}/2$,
  \begin{align}
  \tbfu^{k+1} &= \tbfu^{k} - \eta_{\bfu}\,\bfg_{\tbfu}^k\ ,\\\nonumber
  \tbfu^{k+1}\,\bfq &= \tbfu^{k}\,\bfq - \eta_{\bfu}\,\bfg_{\tbfu}^k\,\bfq\
  ,\\\nonumber
   &= \tbfw^{k} - \eta_{\bfu}\,2\,\bfg_{\tbfw}^k\ ,\\\nonumber
   &= \tbfw^{k} - \eta_{\bfw}\,\bfg_{\tbfw}^k\ ,\\\nonumber
   &= \tbfw^{k+1}\ .
  \end{align}
  Hence, the proof is complete.
\end{enumerate}
\end{proof}

Note that, in the implementation of \gls{BC}, the auxiliary variables $\tbfw$
are clipped between $[-1,1]$ as it does not affect the $\sign$ function.
In the $\bfu$-space, this clipping operation would translate into a projection
to the polytope $\calS$, meaning $\tbfw\in[-1,1]$ implies $\tbfu\in\calS$, where
$\tbfw$ and $\tbfu$ are related according to $\tbfw=\tbfu\bfq$.
Even in this case,~\proref{pro:bc_vs_pgdh1} holds, as the assumption
$\tbfw^k=\tbfu^k\bfq$ is still satisfied.

\subsection{Approximate Gradients through Hardmax}\label{ap:gradhp}
In previous works~\cite{courbariaux2015binaryconnect,rastegari2016xnor}, to
allow back propagation through the $\sign$ function, the
straight-through-estimator~\cite{stehinton} is used. 
Precisely, the partial derivative with respect to the $\sign$ function is
defined as:
\begin{equation}\label{eq:gsign}
\frac{\partial \sign(r)}{\partial r} := \I[|r| \le 1]\ .
\end{equation}
To make use of this, we intend to write the projection
function $\hardmax$ in terms of the $\sign$ function. To this end,
from~\eqref{eq:puarg}, for each $j\in\allweights$,
  \begin{align}
  u^k_{j:-1} &= \left\{\begin{array}{ll}
1 & \mbox{if $\tu^k_{j:-1} - \tu^k_{j:1} \ge 0$}\\
0 & \mbox{otherwise} \end{array} \right.\ ,\\
u^k_{j:1} &= 1- u^k_{j:-1}\ .
  \end{align}
  Hence, the projection $\hardmax(\tbfu^k)$ for each $j$ can be written
  as:
  \begin{align}
  u^k_{j:-1} &= (\sign(\tu^k_{j:-1} - \tu^k_{j:1}) + 1)/2\ ,\\
u^k_{j:1} &= (1 - \sign(\tu^k_{j:-1} - \tu^k_{j:1}))/2\ .
  \end{align}
  Now, using~\eqref{eq:gsign}, we can write:
  \begin{equation}
  \left.\frac{\partial \bfu_j}{\partial \tbfu_j}\right|_{\tbfu_j =
  \tbfu^k_j} =\frac{1}{2}
  \begin{bmatrix}
  \I[|\upsilon^k_j| \le 1] &
  -\I[|\upsilon^k_j| \le 1]\\[0.1cm]
  -\I[|\upsilon^k_j| \le 1] &
  \I[|\upsilon^k_j| \le 1]
  \end{bmatrix}\ ,
  \end{equation}
  where $\upsilon^k_j = \tu^k_{j:-1} - \tu^k_{j:1}$.
   
\section{Experimental Details}\label{ap:expr}
To enable reproducibility, we first give the hyperparameter settings used to
obtain the results reported in the main paper in~\tabref{tab:hyper}.
%Then, we provide additional training curves of different methods
% in~\twofigref{fig:morecurves-c}{fig:morecurves-tiny} for better understanding of
% the convergence behaviour.
%In short, the behaviour observed in the main paper holds.

%%%%%%%%% analysis
\begin{table}[t]
    \centering
    %\small    
    \begin{tabular}{llcc}
        \toprule
        Dataset & Architecture & \gls{PMF} wo
        $\tbfu$ & \gls{PMF} \\
        \midrule
        \multirow{2}{*}{\mnist}
         & \slenet{-300} & $96.74$ & $\textbf{98.24}$ \\
         & \slenet{-5}   & $98.78$ & $\textbf{99.44}$\\
        \midrule
        \multirow{2}{*}{\cifar{-10}}
         & \svgg{-16} 	 & $80.18$ & $\textbf{90.51}$\\
         & \sresnet{-18} & $87.36$ & $\textbf{92.73}$\\
        \bottomrule
    \end{tabular}
    \vspace{1ex}
    \caption{\em
    	Comparison of \gls{PMF} with and without storing the auxiliary variables
    	$\tbfu$. Storing the auxiliary variables and updating them is in fact
    	improves the overall performance. However, even without storing $\tbfu$,
    	\gls{PMF} obtains reasonable performance, indicating the usefulness of our
    	relaxation.}
    \label{tab:pmfa}
\end{table}

\subsection{Proximal Mean-field Analysis}
To analyse the effect of storing the auxiliary variables $\tbfu$
in Algorithm~\myref{1}, we evaluate \gls{PMF} with and without storing $\tbfu$,
meaning the variables $\bfu$ are updated directly. 
The results are reported in~\tabref{tab:pmfa}.
Storing the auxiliary variables and updating them is in fact improves the
overall performance. However, even without storing $\tbfu$, \gls{PMF} obtains
reasonable performance, indicating the usefulness of our continuous relaxation.
Note that, if the auxiliary variables are not stored in \gls{BC}, it is
impossible to train the network as the quantization error in the gradients are
catastrophic and single gradient step is not sufficient to move from one
discrete point to the next.

\subsection{Effect of Annealing Hyperparameter $\beta$}
In Algorithm~\myref{1}, the annealing hyperparameter is gradually increased by a
multiplicative scheme. 
Precisely, $\beta$ is updated according to $\beta = \rho\beta$ for some
$\rho>1$.
Such a multiplicative continuation is a simple scheme suggested in
Chapter~\myref{17} of~\cite{nocedal2006numerical} for penalty methods.
To examine the sensitivity of the continuation parameter $\rho$, we report the
behaviour of \acrshort{PMF} on \cifar{-10} with \sresnet{-18} for various values
of $\rho$ in~\figref{fig:rho}.

\begin{figure}[t]
\begin{center}
\includegraphics[width=0.95\linewidth]{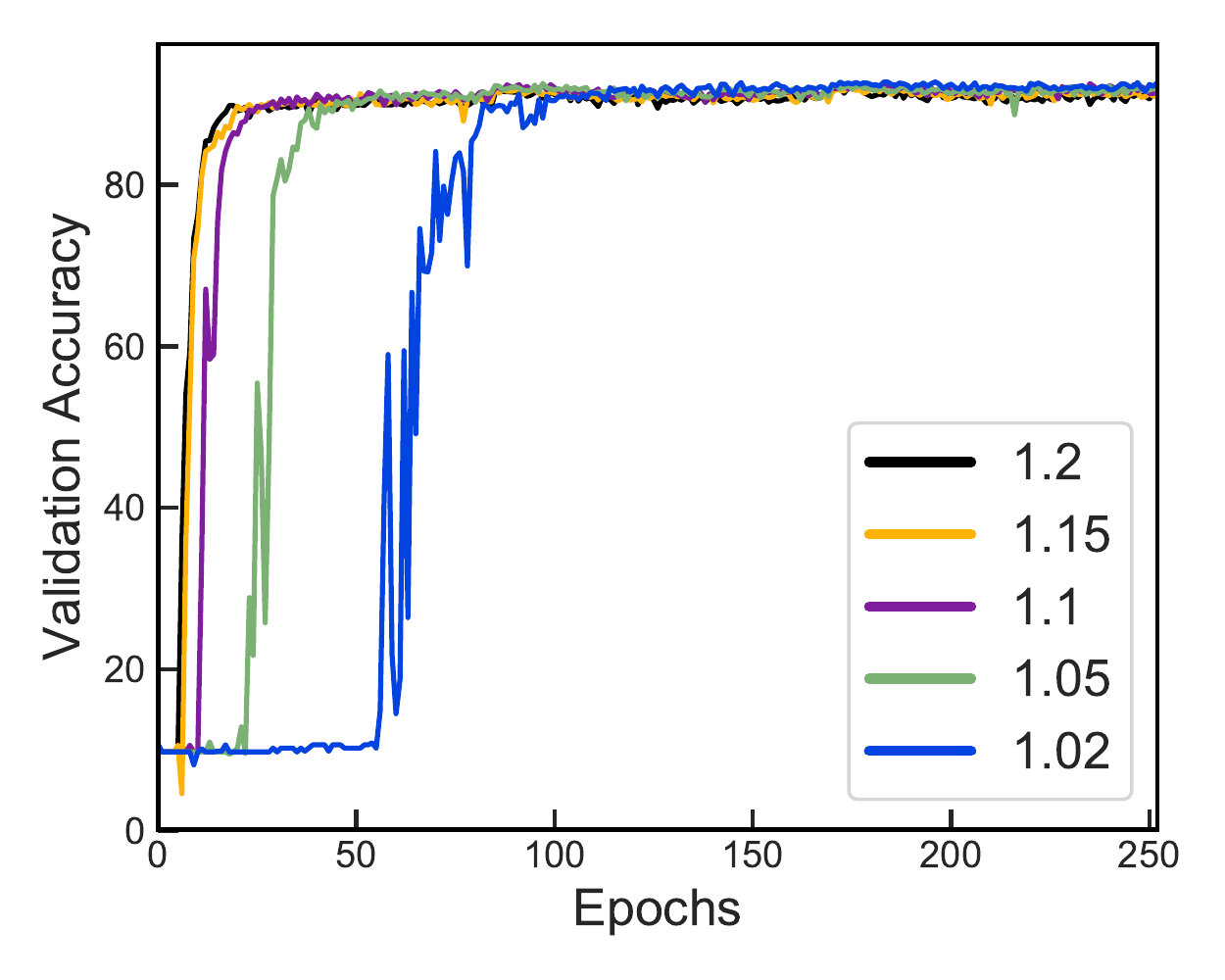}	%trim=[left botm right top]
\vspace{-3ex}
\caption{\em \acrshort{PMF} results on \cifar{10} with \sresnet{-18} by varying
$\rho$ values. 
While \acrshort{PMF} is robust to a range of $\rho$ values, the longer
exploration phase in lower values of $\rho$ tend to yield slightly better final
accuracies.
}
\label{fig:rho}
\end{center}
\end{figure}

\subsection{Multi-bit Quantization}
To test the performance of \acrshort{PMF} for quantization levels beyond binary, we ran \acrshort{PMF} for $2$-bit quantization with
$\calQ = \{−2, −1, 1, 2\}$ on \cifar{-10} with the same hyperparameters as in the binary case and the results are, \sresnet{-18}: $92.88$\%
 and \svgg{-16}: $91.27$\%, respectively. 
 We believe, the improvements over binary ($+0.15$\% and $+0.76$\%) even
without hyperparameter tuning shows the merits of \acrshort{PMF} for NN quantization. Note, similar to existing methods~\cite{bai2018proxquant},
we can also obtain different quantization levels for each weight $w_j$, which would further improve the performance.

%%%%%%% hyperparams
\begin{table*}[t]
\small
    \centering
\begin{tabular}{l|ccccc|ccccc}
\toprule
    \multirow{2}{*}{Hyperparameter} & \multicolumn{5}{c|}{\mnist{} with
\slenet{-300/5}} & \multicolumn{5}{c}{\tinyimagenet{} with \sresnet{-18}}\\
     & \acrshort{REF} & \acrshort{BC}/\acrshort{PICM} & \acrshort{PQ} &
\acrshort{PGD} & \acrshort{PMF} & \acrshort{REF} & \acrshort{BC}/\acrshort{PICM}
& \acrshort{PQ} & \acrshort{PGD} & \acrshort{PMF}\\
\midrule

learning\_rate                                 & 0.001 & 0.001 & 0.01  & 0.001 &
0.001 & 0.1            & 0.0001 & 0.01   & 0.1            & 0.001\\
      lr\_decay                                & step  & step  & -     & step  &
step  & step           & step   & step   & step           & step\\
   lr\_interval                                & 7k    & 7k    & -     & 7k    &
7k    & 60k            & 30k    & 30k 	  & 30k            & 30k\\
      lr\_scale                                & 0.2   & 0.2   & -     & 0.2   &
0.2   & 0.2            & 0.2    & 0.2 	  & 0.2            & 0.2\\  
      momentum                                 & -     & -     & -     & -     &
-     & 0.9            & -      & -      & 0.95           & - \\
     optimizer                                 & Adam  & Adam  & Adam  & Adam  &
Adam  & \acrshort{SGD} & Adam   & Adam   & \acrshort{SGD} & Adam\\
  weight\_decay                                & 0     & 0     & 0     & 0     &
0     & 0.0001         & 0.0001 & 0.0001 & 0.0001         & 0.0001\\
   $\rho$ (ours) or reg\_rate (\acrshort{PQ})  & -     & -     & 0.001 & 1.2   &
1.2   & -              & -      & 0.0001 & 1.01           & 1.02\\

\midrule
     & \multicolumn{5}{c|}{\cifar{-10} with \svgg{-16}} &
\multicolumn{5}{c}{\cifar{-10} with \sresnet{-18}}\\
     & \acrshort{REF} & \acrshort{BC}/\acrshort{PICM} & \acrshort{PQ} &
\acrshort{PGD} & \acrshort{PMF} & \acrshort{REF} & \acrshort{BC}/\acrshort{PICM}
& \acrshort{PQ} & \acrshort{PGD} & \acrshort{PMF}\\
\midrule
 learning\_rate                                & 0.1            & 0.0001 & 0.01 
 & 0.0001 & 0.001  & 0.1            & 0.0001 & 0.01   & 0.1            &
0.001\\
      lr\_decay                                & step           & step   & -    
 & step   & step   & step           & step   & -      & step           & step\\
   lr\_interval                                & 30k            & 30k    & -    
 & 30k    & 30k    & 30k            & 30k    & -      & 30k            & 30k\\
      lr\_scale                                & 0.2            & 0.2    & -    
 & 0.2    & 0.2    & 0.2            & 0.2    & -      & 0.2            & 0.2\\
      momentum                                 & 0.9            & -      & -    
 & -      & -      & 0.9            & -      & -      & 0.9            & -\\
     optimizer                                 & \acrshort{SGD} & Adam   & Adam 
 & Adam   & Adam   & \acrshort{SGD} & Adam   & Adam   & \acrshort{SGD} & Adam\\
  weight\_decay                                & 0.0005         & 0.0001 &
0.0001 & 0.0001 & 0.0001 & 0.0005         & 0.0001 & 0.0001 & 0.0001         &
0.0001\\
   $\rho$ (ours) or reg\_rate (\acrshort{PQ})  & -              & -      &
0.0001 & 1.05   & 1.05   & -              & -      & 0.0001 & 1.01           &
1.02\\

\midrule
     & \multicolumn{5}{c|}{\cifar{-100} with \svgg{-16}} &
\multicolumn{5}{c}{\cifar{-100} with \sresnet{-18}}\\
     & \acrshort{REF} & \acrshort{BC}/\acrshort{PICM} & \acrshort{PQ} &
\acrshort{PGD} & \acrshort{PMF} & \acrshort{REF} & \acrshort{BC}/\acrshort{PICM}
& \acrshort{PQ} & \acrshort{PGD} & \acrshort{PMF}\\
\midrule
 learning\_rate                               & 0.1                  & 0.01     
     & 0.01               & 0.0001               & 0.0001               & 0.1   
              & 0.0001               & 0.01               & 0.1                 
& 0.001\\
      lr\_decay                               & step                 &
multi-step     & -                  & step                 & step               
 & step                 & step                 & step                 & step    
            & multi-step\\
   \multirow{2}{*}{lr\_interval}              & \multirow{2}{*}{30k} & 20k -
80k,     & \multirow{2}{*}{-} & \multirow{2}{*}{30k} & \multirow{2}{*}{30k} &
\multirow{2}{*}{30k} & \multirow{2}{*}{30k} & \multirow{2}{*}{30k} &
\multirow{2}{*}{30k} & 30k - 80k, \\
                                              &                      & every 10k
     &                    &                      &                      &       
              &                      &                      &                   
  & every 10k\\
      lr\_scale                               & 0.2                  & 0.5      
     & -                  & 0.2                  & 0.2                  & 0.1   
              & 0.2                  & 0.2                  & 0.2               
  & 0.5\\
      momentum                                & 0.9                  & 0.9      
     & -                  & -                    & -                    & 0.9   
              & -                    & -                    & 0.95              
  & 0.95\\
     optimizer                                & \acrshort{SGD}       &
\acrshort{SGD} & Adam               & Adam                 & Adam               
 & \acrshort{SGD}       & Adam                 & Adam                 &
\acrshort{SGD}       & \acrshort{SGD}\\
  weight\_decay                               & 0.0005               & 0.0001   
     & 0.0001             & 0.0001               & 0.0001               & 0.0005
              & 0.0001               & 0.0001               & 0.0001            
  & 0.0001\\
   $\rho$ (ours) or reg\_rate (\acrshort{PQ}) & -                    & -        
     & 0.0001             & 1.01                 & 1.05                 & -     
              & -                    & 0.0001               & 1.01              
  & 1.05\\

\bottomrule
\end{tabular}
\vspace{1ex}
    \caption{\em Hyperparameter settings used for the experiments. 
	Here, if {\em lr\_decay == step}, then the learning rate is multiplied by {\em
lr\_scale} for every {\em lr\_interval} iterations.
	On the other hand, if {\em lr\_decay == multi-step}, the learning rate is
multiplied by {\em lr\_scale} whenever the iteration count reaches any of the
milestones specified by {\em lr\_interval}. 
	Here, $\rho$ denotes the growth rate of $\beta$ (refer Algorithm~\myref{1}) and
$\beta$ is multiplied by $\rho$ every 100 iterations.}
    \label{tab:hyper}
\end{table*}

\else
%%%%%%%% ABSTRACT
\begin{abstract}
 \vspace{-1ex}
Compressing large Neural Networks (\acrshort{NN}) by quantizing the parameters, while maintaining the performance is highly desirable due to reduced memory and time complexity.
In this work, we cast \acrshort{NN} quantization as a discrete labelling problem, and 
%and leverage results from the extensively studied \acrshort{MRF} optimization literature.
%Specifically, we 
by examining relaxations, we design an efficient iterative optimization procedure that involves stochastic gradient descent followed by a projection.
We prove that our simple projected gradient descent approach is, in fact, equivalent to a proximal version of the well-known mean-field method.
These findings would allow the decades-old and theoretically grounded research on \acrshort{MRF} optimization to be used to design better network quantization schemes.
Our experiments on standard classification datasets (\mnist{}, \cifar{10/100}, \tinyimagenet{}) with convolutional and residual architectures show that our algorithm obtains fully-quantized networks with accuracies very close to the floating-point reference networks.
%\vspace{-2ex}
 \end{abstract}
 
\vspace{-3ex}
\section{Introduction}
Despite the success of deep neural networks, they are highly overparametrized,
resulting in excessive computational and memory requirements.
Compressing such large networks by quantizing the parameters, while maintaining
the performance, is highly desirable for real-time applications, or for
resource-limited devices. 

In \gls{NN} quantization, the objective is to learn a network while restricting
the parameters to take values from a small discrete set (usually binary)
representing quantization levels. 
This can be formulated as a {\em discrete labelling problem} where each
learnable parameter takes a label from the discrete set and the learning
objective is to find the label configuration that minimizes the empirical loss.
This is an extremely challenging discrete optimization problem because the
number of label configurations grows exponentially with the number of parameters
in the network and the loss function is highly non-convex.
\SKIP{
We show that this problem can equivalently be formulated as solving a sequence
of {\em discrete labelling problem}, 
where each labelling problem assigns a label from the discrete set to the neural
network parameters while minimizing 
the overall empirical loss.
%
%Note that, the number of parameters in a neural network are usually in
%the order of millions and the number of label configurations become
%exponentially large, making them intractable to
%optimize~\cite{nemhauser1988integer}.
Depending on the interactions among the parameters, each discrete optimization
problem will have 
varied time complexity. In fact in most of the cases it is intractable, except
in certain special cases~\cite{nemhauser1988integer}, as the
number of label configurations grows exponentially with the number of
parameters
in the neural network.\NOTE{This para need work!}
}

Over the past 20 years, similar large-scale discrete labelling problems have
been extensively studied under the context of \gls{MRF} optimization, and many
efficient approximate algorithms have been
developed
~\cite{ajanthanphdthesis,blake2011markov,dokania2015parsimonious% 
,mudigonda2008combinatorial,veksler1999efficient,wainwright2008graphical}.
In this work, we take inspiration from the rich literature on \gls{MRF}
optimization, and design an efficient approximate algorithm based on the popular
mean-field method~\cite{wainwright2008graphical} for \gls{NN} quantization.

Specifically, we first formulate \gls{NN} quantization as a discrete labelling
problem.
Then, we relax the discrete solution space to a convex polytope and introduce an
algorithm to iteratively optimize the first-order Taylor approximation of the
loss function over the polytope.
This approach is a (stochastic) gradient descent method with an
additional projection step at each iteration. 
For a particular choice of projection, we show that our method is equivalent to
a proximal version of the well-known mean-field method.
Furthermore, we prove that under certain conditions, our algorithm specializes
to the popular BinaryConnect~\cite{courbariaux2015binaryconnect} algorithm.
%\SKIP{
%Next, similarly to~\cite{chekuri2004linear,kleinberg2002approximation},
%by introducing indicator variables and relaxing them, we convert the
%labelling problem into learning the probability of each parameter taking a
%label
%from the discrete label set.
%Then, by assuming independence between parameters\NOTE{is it implicit in case
%of pgd?}, we design a stochastic version of projected gradient descent and show
%that its update is exactly the same as a proximal version of the mean-field
%method~\cite{wainwright2008graphical}.
%%
%Furthermore, we show that for a particular choice of projection, our algorithm
%specializes to the popular BinaryConnect quantization
%method~\cite{courbariaux2015binaryconnect}. \NOTE{This pare is procedural, need
%better writing style}
%}

The \gls{MRF} view of \gls{NN} quantization opens up many
interesting research directions.
%especially, the connection to mean-field method
%In contrast to the existing \gls{NN} quantization
% methods~\cite{hubara2017quantized,rastegari2016xnor,yin2018binaryrelax}, our
% algorithm provides a new perspective, based on \gls{MRF} optimization, and it
% opens up many interesting research directions. 
In fact, %from the equivalence with the mean-field method, it is evident that 
our approach represents the simplest case where the \gls{NN} parameters are
assumed to be independent of each other. 
However, one can potentially model second-order or even high-order interactions
among parameters and use efficient inference algorithms developed and
well-studied in the \gls{MRF} optimization literature. 
Therefore, we believe, many such algorithms can be transposed into this
framework to design better network quantization schemes.
Furthermore, in contrast to the existing \gls{NN} quantization
methods~\cite{hubara2017quantized,rastegari2016xnor}, we quantize {\em all} the
learnable parameters in the network (including biases) and our formulation can
be seamlessly extended to quantization levels beyond binary.

We evaluate the merits of our algorithm on \mnist{}, \cifar{-10/100}, and
\tinyimagenet{} classification datasets with convolutional and residual
architectures. 
Our experiments show that the quantized networks obtained by our algorithm
yield accuracies very close to the floating-point counterparts while
consistently outperforming directly comparable baselines. %when all learnable
% parameters are quantized.
Our code is available at \href{https://github.com/tajanthan/pmf}
{https://github.com/tajanthan/pmf}.

\SKIP{
Our formulation has several attractive properties:
\begin{tight_itemize}
\item Since it is formulated as an MRF optimization, it allows researchers to
think about dependency between parameters which can result in better neural
network quantization schemes.
\item In contrast to the existing
methods~\cite{hubara2017quantized,rastegari2016xnor,yin2018binaryrelax}, we
quantize all the learnable parameters including biases and it allows us to
extend to quantization levels beyond binary without additional heuristics.
\end{tight_itemize}
}
\SKIP{
\begin{tight_itemize}
\item quantized neural networks are popular and efficient in terms of memory
and
time
\item we formulate it as a labelling problem (discrete optimization)
\item labelling problem is well-studied and MRF literature can be leveraged to
design better quantization algos
\item we design softmax-pgd and show that it is equivalent to proximal
mean-field
\item quantizing all the learnable parameters
\item show binarycoNNect as a special case
\item future: better projection, rounding considering dependency among
parameters
\item can we say existing methods are heuristically designed?
\end{tight_itemize} 
}

%\newpage
\section{Neural Network Quantization}
\acrfull{NN} quantization is the problem of learning neural network parameters
restricted to a small discrete set representing quantization levels.
%Expecting good solution to this problem 
This primarily relies on the hypothesis that overparametrization of
\gls{NN}s makes it possible to obtain a quantized network with performance
comparable to the floating-point network.
To this end, given a dataset $\calD=\{\bfx_i, \bfy_i\}_{i=1}^n$, the \gls{NN}
quantization problem can be written as:
\vspace{-1.2ex}
\begin{align}\label{eq:dnnobj}
\min_{\bfw\in \calQ^m} L(\bfw;\calD) &:=\sum_{i=1}^n
\ell(\bfw;(\bfx_i,\bfy_i))\ .\\[-4.3ex]\nonumber
%\bfw &\in \calQ^m\ .
\end{align}
Here, $\ell(\cdot)$ is the input-output mapping composed with a standard loss
function (\eg, cross-entropy loss), $\bfw$ is the $m$ dimensional parameter
vector, and $\calQ$ with $|\calQ|=d$ is a predefined discrete set representing
quantization levels (\eg, $\calQ=\{-1,1\}$ or $\calQ=\{-1,0,1\}$). 
In \eqref{eq:dnnobj}, we seek a {\em fully-quantized network} where all the
learnable parameters including biases are quantized. This is in contrast to the
previous methods~\cite{courbariaux2015binaryconnect,rastegari2016xnor} where
some parts of the network are not quantized (\eg, biases and last layer
parameters).  
%This also eliminates the need for additional heuristics. 

\SKIP{
\NOTE{move this to related work or beginning of experiments}
In this work, we restrict our attention to quantizing weights, however,
activations can also be quantized similarly to the previous
works~\cite{hubara2017quantized,rastegari2016xnor}. In addition to that,
layer-wise scalars can also be learned as
in~\cite{rastegari2016xnor,yin2018binaryrelax}, but we believe, that goes
beyond the scope of this paper.
}

\subsection{NN Quantization as Discrete Labelling}
\gls{NN} quantization~\plaineqref{eq:dnnobj} naturally takes the form of a {\em
discrete labelling problem} where each learnable parameter $w_j$ takes a label
$q_\lambda$ from the discrete set $\calQ$. 
In particular, \eqref{eq:dnnobj} is directly related to an \gls{MRF}
optimization problem~\cite{mrf1980kindermann} where the random variables
correspond to the set of weights $\bfw$, the label set is $\calQ$, and the
energy function is $L(\bfw)$.
We refer to Appendix~\myref{A} for a brief overview on \gls{MRF}s.

An important part of an \gls{MRF} is the factorization of the energy function
that depends on the interactions among the random variables. 
While modelling a problem as an \gls{MRF}, emphasis is given to the form of
the energy function (\eg, submodularity) as well as the form of the interactions
(cliques), because both of these aspects determine the complexity of the
resulting optimization.
In the case of \gls{NN}s, the energy function (\ie, loss) is a composition of
functions which, in general, has a variety of interactions among the random
variables.
For example, a parameter at the initial layer is related to parameters at
the final layer via function composition. 
Thus, the energy function does not have an explicit factorization. %na\"ively
In fact, optimizing~\eqref{eq:dnnobj} directly is intractable due to the
following inherent problems~\cite{kolmogorov2004energy,nemhauser1988integer}: 
\vspace{-2ex}
\begin{tight_enumerate}
\item The solution space is discrete with exponentially many feasible points
($d^m$ with $m$ in the order of millions).
\item The loss function is highly non-convex and does not satisfy any regularity
condition (\eg, submodularity).
\item The loss function does not have an explicit factorization (corresponding to
a neighbourhood structure).
\end{tight_enumerate}
% making it computationally infeasible to directly optimize the
% objective~\plaineqref{eq:dnnobj}.
This hinders the use of any off-the-shelf discrete optimization algorithm.
However, to tackle the aforementioned problems, we take inspiration from the
\gls{MRF} optimization literature~
\cite{besag1986statistical,chekuri2004linear,wainwright2008graphical}.
In particular, we first relax the discrete solution space to a convex polytope
and then iteratively optimize the first-order approximation of the loss
over the polytope. %The first-order approximation inherently avoids any
% interaction among the weight parameters.
Our approach, as will be shown subsequently, belongs to
the class of (stochastic) gradient descent methods and is applicable to any loss
function. 
Next we describe these relaxations and the related optimization in detail.

\SKIP{
Note, there are two major problems with this objective: (1) as opposed to the
standard labelling problems, the loss function in this case is highly
non-convex
with no specific standard form making it impossible to use off-the-shelf
efficient inference algorithms; and (2) the solution space is discrete,
requiring exhaustive search. To this end, we propose relaxations for both these
issues, making our approach similar to \gls{SGD} type optimization algorithm.
From very high-level point of view, we first relax the solution space and make
it continuous, then we optimize the first-order Taylor's approximation of the
loss function over this relaxed space. We repeat this process until
convergence.
}

\subsection{Continuous Relaxation of the Solution Space}\label{sec:uspace}
Recall that $\calQ$ is a finite set of $d$ real-valued parameters.
The elements of $\calQ$ will be indexed by $\lambda\in\alllabels$. 
An alternative representation of $\calQ$ is by a $d$-dimensional vector
$\bfq$ with entries $q_\lambda \in \calQ$.
A element $w\in\calQ$ can be written in terms of indicator variables $u_\lambda \in 
\{0, 1\}$ as $w = \sum_{\lambda=1}^d q_\lambda u_\lambda$, assuming
that only one value of $q_\lambda$ has value $1$.
Denote by $\calV$ the set of size $d$ of such $d$-vectors with a single $1$ component
(elements of the standard basis of $\R^d$)
 acting as indicator vectors for the elements of $\calQ$.
 Explicitly, a vector $\bfu_j \in \R^d$ is in set $\calV$ if 
 \vspace{-1ex}
 \begin{align*}
  \sum_{\lambda=1}^d u_{j:\lambda} = 1  \mbox{~~ and ~~}
u_{j:\lambda}  \in \{0,1\}\quad\,\forall\,\lambda\in\alllabels\ .\\[-4ex]\nonumber
 \end{align*}
Similarly, the vector \( \bfw \in \calQ^m \) of all parameters 
can be represented
using indicator
variables as follows. Let ${u_{j:\lambda}\in\{0,1\}}$ be the indicator variable,
where $u_{j:\lambda}=1$ if and only if $w_j = q_\lambda \in \calQ$. 
Then, for any $j\in \allweights$, we can write
\vspace{-1ex}
\begin{align}\label{eq:indic}
w_j &= \sum_{\lambda=1}^{d}  u_{j:\lambda}\,q_\lambda = \big< \bfu_j, \bfq \big>
\quad \mbox{where}\quad \bfu_j \in \calV \ .\\[-4ex]\nonumber
\end{align}
Any $w_j$ represented using \eqref{eq:indic} belongs to $\calQ$. 
The vector $\bfw$ of all parameters
may be written as a matrix-vector product,
\vspace{-3ex}
\begin{align}\label{eq:indicv}
\bfw &= \bfu\bfq 
\mbox{   ~~ where  }\bfu \in \calV^m \ .\\[-4ex]\nonumber
%&= \left\{\begin{array}{l|l}
%    \multirow{2}{*}{$\bfu$} & \sum_{\lambda=1}^{d} u_{j:\lambda} = 1,
%  \quad\forall\,j\\\nonumber 
%  &u_{j:\lambda} \in \{0,1\}, \quad\forall\,j, \lambda
%  \end{array} \right\}~.
  \end{align}
Here, $\bfu = \{u_{j:\lambda}\}$ is thought of as an $m \times d$ matrix 
% \footnote{To simplify the notation, we
%denote $\bfu$ as a matrix flattening of which will give an $md$ dimensional
%vector.}
(each row $\bfu_j$, for $j\in\allweights$ is an element of $\calV$).
  Note that there is a one-to-one correspondence between the sets $\calV^m$ and
$\calQ^m$.
  Substituting \eqref{eq:indicv} in the \gls{NN} quantization
objective~\plaineqref{eq:dnnobj} results in the variable change from $\bfw$ to
$\bfu$ as:
\vspace{-1ex}
  \begin{align}
  \label{eq:labelobj}
  \min_{\bfw\in\calQ^m} L(\bfw;\calD) = \min_{\bfu \in \calV^m}
  %\sum_{i=1}^n \ell(\bfu\bfq;(\bfx_i, \bfy_i))\ .% \nonumber
  L(\bfu\bfq; \calD) \ .\\[-4ex]\nonumber
  \end{align}
  Even though the above variable change converts the problem from $m$ to $md$
dimensional space, 
the cardinalities of the sets $\calQ^m$ and $\calV^m$ are the
same. The binary constraint $u_{j:\lambda} \in \{0,1\}$ together with the
  non-convex loss function $L(\cdot)$ makes the problem
  NP-hard~\cite{nemhauser1988integer}.
  
  \vspace{-2ex}
  \paragraph{Relaxation.}
  By relaxing the binary constraints to ${u_{j:\lambda} \in [0,1]}$, instead of
  $ u_{j:\lambda} \in \{0,1\}$ we obtain
the convex hull $\calS$ of the set $\calV^m$.  The minimization 
  \eqref{eq:labelobj} may be carried out over $\Delta^m$ instead of
  $\calV^m$.  In detail, we define
 \begin{equation}
  \Delta=\left\{\begin{array}{l|l} \multirow{2}{*}{$\bfz\in\R^d$} &
\sum_{\lambda}
  z_{\lambda} = 1\\ &z_{\lambda} \ge 0, \quad\forall\, \lambda
  \end{array} \right\} \ .
 \end{equation}
 This is the standard $(d\!-\!1)$-dimensional simplex embedded in $\R^d$
 and the vertices of $\Delta$
 are the points in $\calV$.  Similarly, the Cartesian product $\Delta^m$ is
 the convex hull of $\calV^m$, which are in turn the vertices of $\Delta^m$.
 
 Simplex $\Delta$ will be referred to as the {\em probability simplex} because an
 element $u\in\Delta$ may be thought of (formally) as a probability distribution
 on the finite set $\{ 1, \ldots, d\}$.  A value $u_{\lambda}$ is the probability of
choosing the discrete parameter $w = q_\lambda \in \calQ$.  
With this probabilistic interpretation,
one verifies that $\bfu \bfq = \E_{\bfu}[\bfw]$, the expected value of 
the vector of parameters $\bfw$, where each $w_j$ has independent 
probability distribution defined by $\bfu_j$.
  %
%  \begin{equation}
%  \calS = \conv(\calV^m) = \left\{\begin{array}{l|l}
%  \multirow{2}{*}{$\bfu$} & \sum_{\lambda} u_{j:\lambda} = 1, \quad\forall\,j\\
%  &u_{j:\lambda} \ge 0,\ \ \ \quad\quad\forall\,j, \lambda \end{array} \right\}\
%.\
%  \end{equation}
%  %
%  %In particular, $\calS$ is called the \textit{marginal polytope} in the
%  % graphical
%  %models literature~\cite{wainwright2008graphical}. 
%  Note that the set $\calS$ decomposes over each $j$, and it is in fact the
%  Cartesian product of probability simplexes $\Delta$ of dimension $d$. Thus,
%  %
%  \begin{equation}\label{eq:decoms}
%  \calS = \Delta \otimes \ldots \otimes \Delta \mbox{ where
%  $\ \Delta=\left\{\begin{array}{l|l} \multirow{2}{*}{$\bfz$} &
%\sum_{\lambda}
%  z_{\lambda} = 1\\ &z_{\lambda} \ge 0, \quad\forall\, \lambda
%  \end{array} \right\}$} ~.
%  \end{equation} 
%  %

%  Therefore, for a point $\bfu\in \calS$, the vector $\bfu_j$ for each
%  $j$ ($j$-th row of matrix ${\bfu}$) belongs to the probability simplex of
%dimension $d$. 
  %
  Now, the relaxed optimization can be written as:
  \vspace{-1ex}
  \begin{align}
  \label{eq:simobj}
  \min_{\bfu \in\calS} \tL(\bfu;\calD) &:= L(\bfu\bfq;\calD) \ ,\\[-4ex]\nonumber
  \end{align}
  The minimum of this problem will generally be less than the
  minimum of  \eqref{eq:labelobj}.
  However,  if $\bfu\in\calV^m$, then the 
  loss function $\tL(\bfu)$ has the same value as the original loss function
  $L(\bfw)$. 
  Furthermore, the relaxation of $\bfu$ from $\calV^m$ to $\calS$ translates into
relaxing $\bfw$ from $\calQ^m$ to the convex region $[q_{\min}, q_{\max}]^m$.
  Here, $q_{\min}$ and $q_{\max}$ represent the minimum and maximum quantization
levels, respectively.

  In fact, $\bfu\in\calS$ is an overparametrized representation of
$\bfw\in[q_{\min}, q_{\max}]^m$, and the mapping $\bfu\to\bfw=
\bfu\bfq$ is a many-to-one {\em surjective}
%
%  If the reader does not know what a surjective mapping is, then they
% can look it up.  No need to insult the reader's intelligence.  RIH
%\footnote{A mapping
%$f(x): \calX\to \calY$ is surjective if $\forall\,y\in\calY$,
%$\exists\,x\in\calX$ such that $f(x) = y$.
%}
 mapping. In the case where $d=2$ (two quantization levels), the mapping is one-to-one
and subjective.
%
%However, this representation has an interesting probabilistic interpretation
%that {\em learning $\bfu$ can be interpreted as learning a discrete probability
%distribution over the \gls{NN} parameters $\bfw$}.
%  This interpretation would be useful in drawing the connection between our
%algorithm and the mean-field method later in~\secref{sec:mean-field}.
In addition it can be shown that any local minimum of~\eqref{eq:simobj} (the
relaxed $\bfu$-space) is also a local minimum of the loss in $[q_{\min},
q_{\max}]^m$ (the relaxed $\bfw$-space) and vice versa
(\proref{pro:local-minima}).
  This essentially means that the variable change from $\bfw$ to $\bfu$ does not
alter the optimization problem and a local minimum in the $\bfw$-space can be
obtained by optimizing in the $\bfu$-space.  
  \begin{pro}\label{pro:local-minima}
  %A point $\bfu^k\in\calS$ is a local minimum of $\tL$, if and only if
% $\bfw^k=\bfu^k\bfq$ is a local minimum of $L$ in the region $[q_{\min},
% q_{\max}]^m$.
  Let $f(\bfw) : [q_{\min}, q_{\max}]^m \rightarrow \R$ be a  function,
  and $\bfw$ a point in $[q_{\min}, q_{\max}]^m$ such that 
${\bfw=g(\bfu) = \bfu\bfq}$. Then %a point $\bfu\in\calS$
$\bfu$  is a local
minimum of $f\circ g$ in $\calS$ if and only if $\bfw$ is a local minimum of
$f$ in $[q_{\min}, q_{\max}]^m$.
  \end{pro}
  \begin{proof}
 The function $g: \calS\to [q_{\min}, q_{\max}]^m$ 
 is surjective continuous and affine.  It follows that it is also an open
 map.  From this the result follows easily. %See Appendix~\myref{A}.
  \end{proof} 

  Finally, we would like to point out that the relaxation used while moving from
$\bfw$ to $\bfu$ space is well studied in the \gls{MRF} optimization literature
and has been used to prove bounds on the quality of the
  solutions~\cite{chekuri2004linear,kleinberg2002approximation}.
  In the case of \gls{NN} quantization, in addition to the connection to
mean-field (\secref{sec:mean-field}), we believe that this relaxation allows for
exploration, which would be useful in the stochastic setting.

  \subsection{First-order Approximation and Optimization}
  %Following the standard practices, at a given iterate, we optimize the
% first-order approximation of the relaxed \gls{NN} quantization
% objective~\plaineqref{eq:simobj} using \gls{PGD}. 
  %Recall, the constraints now form a convex polytope $\calS$. 
  %In Sec.~\ref{sec:uspace} we first mapped the problem from $\bfw \in \calQ^m$
% to $\bfu \in \calV^m$ space, relaxed $\calV^m$ to convex polytope $\calS$, and
% showed how this transformation does not alter the optimization problem. 
  Here we talk about the optimization of $\tL(\bfu)$ over $\calS$, discuss how
our optimization scheme allows exploration in the parameter space, and also
discuss the conditions when this optimization will lead to a quantized solution
in the $\bfw$ space, which is our prime objective. 

  \gls{SGD}\footnote{The difference between \gls{SGD} and gradient descent is
that the gradients are approximated using a stochastic oracle in the former
case.}~\cite{robbins1985stochastic} is the de facto method of choice for
optimizing neural networks.
  In this section, we interpret \gls{SGD} as a proximal method, which will be
useful later to show its difference to our final algorithm. %which resembles a
% proximal version of the mean-field method. 
  In particular, \gls{SGD} (or gradient descent) can be interpreted as
iteratively minimizing the first-order Taylor approximation of the loss function
augmented by a proximal term~\cite{parikh2014proximal}.
  In our case, the objective function is the same as \gls{SGD} but the feasible
points are now constrained to form a convex polytope. 
  %Our objective is exactly the same with additional constraints that requires
% the parameters to always lie over the polytope. 
  Thus, at each iteration $k$, the first-order objective can be written as:
  %\vspace{-0.3cm}
  \vspace{-1ex}
  \begin{align}
  %\begin{split}
  \label{eq:taylor}
  \bfu^{k+1} &= \amin{\bfu\in\calS}\ \tL(\bfu^k) + \left\langle \bfg^{k},
\bfu-\bfu^k\right\rangle\fro + \frac{1}{2\eta}\left\|\bfu-\bfu^k\right\|^2\fro\ ,
\nonumber\\[-1ex] 
&= 
\amin{\bfu\in\calS}\, \left\langle \bfu,\, \eta \bfg^{k}- \bfu^k \right\rangle\fro + \|\bfu\|^2\fro/2\ ,\\[-5ex]\nonumber
%\end{split}
  \end{align}
  where $\eta>0$ is the learning rate and $\bfg^k := \nabla_{\bfu}\tL^{k}$ is
the
  stochastic (or mini-batch) gradient of $\tL$ with respect to $\bfu$ evaluated
at
  $\bfu^k$. 
%  Here, $\langle \cdot, \cdot \rangle\fro$ is the Frobenius inner
%product\footnote{This is equivalent to vectorizing the matrices and applying the
%standard inner product.} and $\|\cdot\|\fro$ is the Frobenius norm, respectively.
  %
  In the unconstrained case, by setting the derivative with respect to
  $\bfu$ to zero, one verifies that the above formulation leads to
standard \gls{SGD} updates
$\bfu^{k+1} = \bfu^k - \eta \bfg^k$.
  %Interpreting \gls{SGD} as a proximal method will be useful to show its
% difference to our final algorithm which resembles a proximal version of the
% mean-field method.
  %
  For constrained optimization (as in our case~\plaineqref{eq:taylor}), it
  is natural to use the stochastic version of
\gls{PGD}~\cite{rosasco2014convergence}. 
  Specifically, at iteration $k$, the projected stochastic gradient update can
be written as:
  \vspace{-1ex}
  \begin{equation}
  \label{eq:proj}
  \bfu^{k+1} = \projS\left(\bfu^{k} - \eta\,\bfg^{k}\right)\ ,
  %\vspace{-1ex}
  \end{equation}
  where $\projS(\cdot)$ denotes the projection to the polytope $\calS$. 
  %Note that, except the projection step, this is the same as the standard
  %first-order optimization algorithms, thus, any off-the-shelf optimizer (\eg,
  %Adam~\cite{kingma2014adam}) can be employed.
  Even though this type of problem can be optimized using projection-free
  algorithms~\cite{ajanthan2016efficient,frank1956algorithm,lacoste2012block},
by
  relying on \gls{PGD}, we enable the use of any off-the-shelf first-order
optimization algorithms (\eg,
  Adam~\cite{kingma2014adam}). 
  Furthermore, for a particular choice of projection, we show that the \gls{PGD}
update is equivalent to a proximal version of the mean-field method.

  \SKIP{
  \NOTE{after softmax projection}
  Moreover, we would like to point out that, \gls{PGD}
update~\plaineqref{eq:proj} in fact yields an approximate solution
to~\eqref{eq:taylor}. 
  However, in~\secref{sec:mean-field}, we show that this \gls{PGD} update is
exactly minimizing a proximal version of the well-known mean-field objective.
  %{\em In addition, in ~\secref{sec:mean-field}, we show how the
% update~\plaineqref{eq:proj} is
  %related to the fixed point solution of the well known mean-field objective.}
  }

  \SKIP{
  In case of constrained optimization (as in~\eqref{eq:simobj}), it is natural
to
  use the stochastic version of \gls{PGD}~\cite{rosasco2014convergence}. 
  Specifically, in our case, at iteration $k$, the projected stochastic gradient
update can be written as:
  \begin{align}
  \label{eq:proj}
  \bfu^{k+1} &= \projS\left(\bfu^{k} - \eta\,\bfg^{k}\right)\ ,
  \end{align}
  where $\eta>0$ is the learning rate, $\bfg^k := \nabla_{\bfu}\tL^{k}$ is the
  stochastic (or mini-batch) gradient of $\tL$ with respect to $\bfu$ evaluated
at
  $\bfu^k$, and $\projS(\cdot)$ represents the projection to the polytope. Note
  that, except the projection step, this is the same as the standard
first-order
  Taylor approximation based optimization algorithms, thus, any off-the-shelf
  optimizer (\eg, Adam~\cite{kingma2014adam}) can be employed.
  }

\vspace{-1ex}
  \subsubsection{Projection to the Polytope $\calS$} \label{sec:sm}
 Projection to $\calS$ can be decomposed into $m$ independent
projections to the $d$-dimensional probability simplexes.
 The objective function~\plaineqref{eq:taylor} is also separable for each
$j$.
%, the \acrshort{PGD} algorithm has an assumption that the probability of
%parameter $w_j$ taking a label $\lambda$ (represented by $u_{j:\lambda}$) is
%independent for each $j$.
%  Such an independence assumption is ubiquitous in \gls{NN} literature due to
%its computational efficiency~\cite{amari1998natural,duchi2011adaptive}.  
%  %\textcolor{red}{This independence assumption, even if sounds too relaxed, has
%% been used extensively in \gls{NN} literature cite Fisher etc. for efficient
%% computation.}
  Thus, for notational convenience, without loss of generality, we assume
$m=1$.
  Now, for a given updated parameter $\tbfu^{k+1} = \bfu^k - \eta\, \bfg^k$
(where $\tbfu^{k+1}\in\R^d$), we discuss three approaches of projecting to the
probability simplex $\Delta$.
  %Thus, for a given $d$-dimensional vector $\tbfv\in\R^d$, we discuss three
% approaches of projecting into the probability simplex $\Delta$. 
  An illustration of these projections is shown in~\figref{fig:explor}. 
  In this section, for brevity, we also ignore the superscript $k+1$.
   
 \begin{figure}
\begin{center}
\includegraphics[width=\linewidth, trim=4.2cm 4.4cm 8.5cm 2.8cm, clip=true,
page=4]{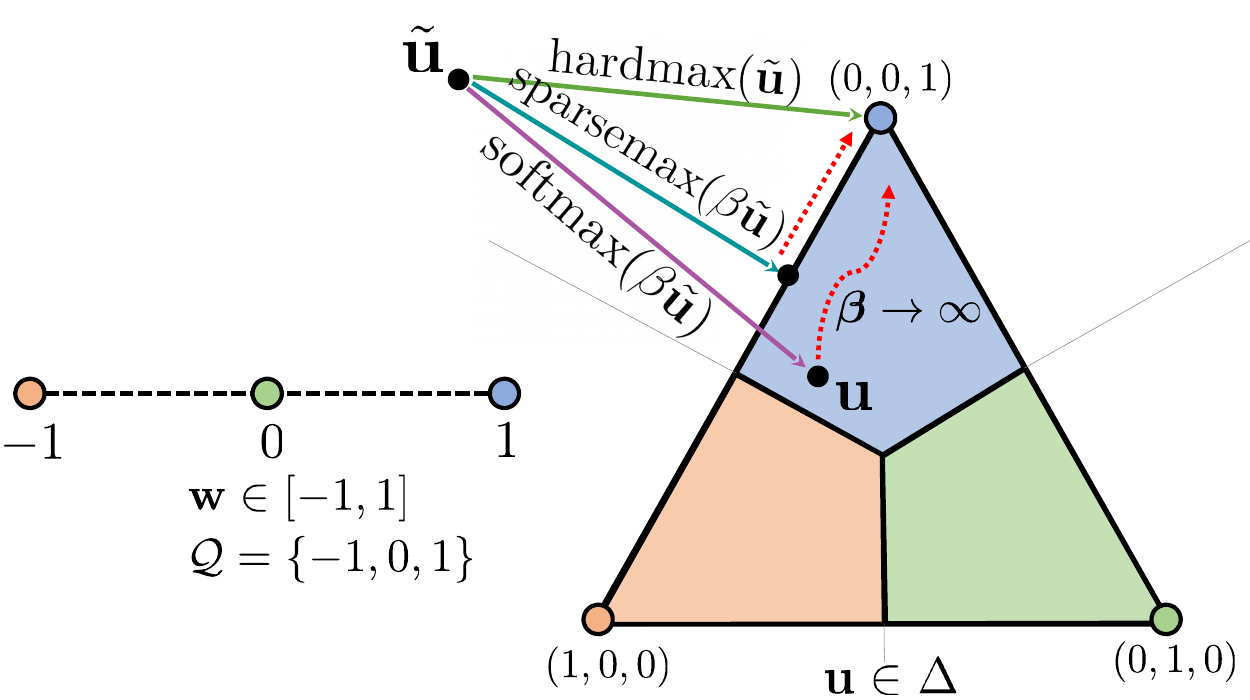}	%trim=[left botm right top]
\vspace{-4ex}
\caption{\em Illustration of $\bfw$ and $\bfu$-spaces, different projections,
and exploration with $\softmax$ when $m=1$.
Here each vertex of the simplex corresponds to a discrete quantization level in
the $\bfw$-space and the simplex is partitioned based on its vertex
association.
Given an infeasible point $\tbfu$, it is projected to the simplex via $\softmax$
(or $\sparsemax$) and when $\beta\to\infty$, the projected point would move
towards the associated vertex.
 }
\label{fig:explor}
\end{center}
\vspace{-5ex}
\end{figure}
  
\vspace{-1ex}
\paragraph{Euclidean Projection (Sparsemax).}
The standard approach of projecting to a set in the Euclidean space is via
$\sparsemax$~\cite{martins2016sparsemax}.
Given a scalar $\beta>0$ (usually $\beta =1$), $\sparsemax$ amounts to finding a
point $\bfu$ in $\Delta$ which is the closest to $\beta\tbfu$, namely
\vspace{-1ex}
\begin{align}\label{eq:sparsemax}
\bfu &= \sparsemax(\beta\tbfu) = \amin{\bfz\in\Delta}\,
\left\|\bfz-\beta\tbfu\right\|^2\ .\\[-4.5ex]\nonumber
\end{align}
%For brevity, we use $\bfu=\sparsemax({\beta\tbfu})$ to denote $\sparsemax$
% applied for each $j\in\allweights$ independently.
As the name suggests, this projection is likely to hit the boundary of the
simplex\footnote{Unless $\beta\tbfu$ when projected to the simplex plane is in
$\Delta$, which is rare.}, resulting in sparse solutions ($\bfu$) at every
iteration. 
Please refer to~\cite{martins2016sparsemax} for more detail. 
As $\beta$ increases, the projected point moves towards a vertex. %, and as
% $\beta \rightarrow \infty$, it is likely to hit a vertex. 
%\textcolor{red}{give the form of the gradients?}

\vspace{-1ex}
\paragraph{Hardmax Projection.} 
The $\hardmax$ projection maps a given $\tbfu$
to one of the vertices of the simplex $\Delta$:
\vspace{-1ex}
\begin{align}\label{eq:hm}
\bfu &= \hardmax(\tbfu)\ ,\\\nonumber
u_{\lambda} &= \left\{\begin{array}{ll}
1 & \mbox{if $\lambda = \amax{\mu\in\calQ}\, \tu_{\mu}$}\\
0 & \mbox{otherwise} \end{array} \right. \quad \mbox{for} \quad \lambda\in\alllabels\
.\\[-4.5ex]\nonumber
\end{align}
%
%Note, by design, $\hardmax$ finds extremely sparse projections and it is easy to
%show that $\hardmax$ is exactly the Euclidean projection to the set of vertices
%of the simplex.\\ 
%Furthermore, in Appendix~\myref{E.1} we show that the gradients through
% $\hardmax$ in the binary case can be approximated using the
% straight-through-estimator~\cite{stehinton}. 

\vspace{-2ex}
\paragraph{Softmax Projection.}
We now discuss the $\softmax$ projection which projects a point to the interior
of the simplex, leading to dense solutions. 
Given a scalar $\beta>0$, the $\softmax$ projection is:
\vspace{-2ex}
\begin{align}
\label{eq:spgdup}
\bfu &= \softmax(\beta\tbfu)\ , \\ \nonumber
u_{\lambda} &= \frac{\exp(\beta \tu_{\lambda})}{\sum_{\mu  \in \calQ} \; \exp(\beta
\tu_{\mu})}\quad \forall\, \lambda\in\alllabels\ .\\[-4.5ex]\nonumber
\end{align}
%It is easy to verify that \( \bfu\in\Delta \). 
%Similarly, we use $\bfu=\softmax({\beta\tbfu})$ to denote $\softmax$ applied to
% each $j\in\allweights$ independently.
Even though approximate in the Euclidean sense, $\softmax$ shares many desirable
properties to $\sparsemax$~\cite{martins2016sparsemax} (for example, it preserves the
relative order of $\tbfu$) and when $\beta\to\infty$, the projected point moves
towards a vertex.

\vspace{-1ex}
\subsubsection{Exploration and Quantization using Softmax}
\vspace{-1ex}
All of the projections discussed above are valid in the sense that the projected
point lies in the simplex $\Delta$. 
However, our goal is to obtain a quantized solution in the $\bfw$-space which is
equivalent to obtaining a solution $\bfu$ that is a  vertex of the simplex
$\Delta$.
Below we provide justifications behind using $\softmax$ with a monotonically
increasing schedule for $\beta$ in realizing this goal, rather than either
$\sparsemax$ or $\hardmax$ projection.

Recall that the main reason for relaxing the feasible points to lie within the
 simplex $\Delta$ is to simplify the optimization problem with the hope
that optimizing this relaxation will lead to a better solution.
However, in case of $\hardmax$ and $\sparsemax$ projections, the effective
solution space is restricted to be either the set of vertices $\calV$ (no
relaxation) or the boundary of the simplex (much smaller subset of $\Delta$).
Such restrictions hinder exploration over the simplex and do not fully utilize
the potential of the relaxation. 
In contrast, $\softmax$ allows {\em exploration} over the entire simplex and a
monotonically increasing schedule for $\beta$ ensures that the solution
gradually approaches a vertex.
This interpretation is illustrated in~\figref{fig:explor}. 

\vspace{-2ex}
\paragraph{Entropy based view of Softmax.}
In fact, $\softmax$ can be thought of as a ``noisy'' projection to the set of
vertices $\calV$, where the noise is controlled by the hyperparameter $\beta$.
We now substantiate this interpretation by providing an entropy based view for
the $\softmax$ projection.
\vspace{-1ex} 
\begin{lem}\label{lem:sm}
Let $\bfu = \softmax(\beta\tbfu)$ for some $\tbfu\in\R^d$ and $\beta>0$. Then,
\vspace{-1.5ex}
\begin{equation}\label{eq:sm_entropy}
\bfu = \amax{\bfz\in\Delta}\ \left\langle\tbfu, \bfz\right\rangle\fro +
\frac{1}{\beta}H(\bfz)\ ,
\vspace{-1ex}
\end{equation} 
where ${H(\bfz) = -\sum_{\lambda=1}^{d}z_{\lambda}\,\log z_{\lambda}}$ is the
entropy.
\end{lem}
\begin{proof}
This can be proved by writing the Lagrangian and setting the derivatives to
zero.
%See Appendix~\myref{B}.
\end{proof}
\vspace{-1ex}
The $\softmax$ projection translates into an
entropy term in the objective function~\plaineqref{eq:sm_entropy}, and for small
values of $\beta$, it allows the iterative procedure to explore the optimization
landscape.
%For the $\softmax$ projection, an entropy based view of exploration is provided
% in~\secref{sec:mean-field}.
We believe, in the stochastic setting, such an explorative behaviour is crucial,
especially in the early stage of training. 
Furthermore, our empirical results validate this hypothesis that \acrshort{PGD}
with $\softmax$ projection is relatively easy to train and yields consistently
better results compared to other \acrshort{PGD} variants.
Note that, when $\beta\to \infty$, the entropy term vanishes and $\softmax$ approaches $\hardmax$.

Note, constraining the solution space through a hyperparameter ($\beta$ in our
case) has been extensively studied in the optimization literature and one such
example is the barrier method~\cite{boyd2009convex}. 
Moreover, even though the $\softmax$ based \gls{PGD} update yields an
approximate solution to~\eqref{eq:taylor}, in~\secref{sec:mean-field}, we prove
that it is theoretically equivalent to a proximal version of the mean-field
method.

\section{Softmax based PGD as Proximal Mean-field}
%\vspace{-1ex}
\label{sec:mean-field}
\NOTE{make sure the flow is coherent, tone down connection to MRF and add a
discussion on entropic penalty} 
Here we discuss the connection between $\softmax$ based \gls{PGD} and the
well-known mean-field method~\cite{wainwright2008graphical}. 
Precisely, we show that the update $\bfu^{k+1} = \softmax({\beta(\bfu^k
-\eta\,\bfg^k)})$ is actually an {\em exact fixed point update} of a modified
mean-field objective function. 
This connection bridges the gap between the \gls{MRF} optimization and the
\gls{NN} quantization literature.
%In the scope of this section we consider that the gradients are
%computed with the full dataset, \ie, the non-stochastic setting.
\SKIP{
For some $\beta>0$, $\eta>0$ and $j\in\allweights$, the $\softmax$
based \gls{PGD} update~\plaineqref{eq:proj} can be written as:
\begin{equation}\label{eq:spgdup}
\bfu^{k+1}_{j:\lambda} = \frac{e^{\beta
\left(u^{k}_{j:\lambda}-\eta\,g^{k}_{j:\lambda}\right)}}{\sum_{\mu}e^{\beta
\left(u^{k}_{j:\mu}-\eta\,g^{k}_{j:\mu}\right)}}\quad \forall\, \lambda\ .
\end{equation}
}
%\NOTE{A bit on mean-field here and refer to appendix}
We now begin with a brief review of the mean-field method and then proceed with
our proof.

\vspace{-2ex}
\paragraph{Mean-field Method.}
A self-contained overview is provided in Appendix~\myref{A}, but here we review
the important details.
Given an energy (or loss) function $L(\bfw)$ and the corresponding probability
distribution of the form $P(\bfw)=e^{-L(\bfw)}/Z$, mean-field approximates
$P(\bfw)$ using a fully-factorized distribution $U(\bfw)=\prod_{j=1}^m
U_{j}(w_j)$.
Here, the distribution $U$ is obtained by minimizing the
\acrshort{KL}-divergence $\kl{U}{P}$.
Note that, from the probabilistic interpretation of $\bfu\in\calS$ (see
\secref{sec:uspace}), for each $j\in\allweights$, the probability
$U_j(w_j=q_\lambda)=u_{j:\lambda}$.
Therefore, the distribution $U$ can be represented using the variables
$\bfu\in\calS$, and hence, the mean-field objective can be written as:
\vspace{-1ex}
%Finally, the minimum of $L(\bfw)$ can be obtained by taking the elements $w_j$
% for $j\in\allweights$ that maximize the distributions $U_j$.
\begin{equation}\label{eq:mf}
\amin{\bfu\in\calS}\ \kl{\bfu}{P} = \amin{\bfu\in\calS}\ \E_{\bfu}[L(\bfw)] -
H(\bfu)\ ,
\vspace{-1ex}
\end{equation}
where $\E_{\bfu}[\cdot]$ is expectation over $\bfu$ and $H(\bfu)$ is the
entropy.

In fact, mean-field has been extensively studied in the \gls{MRF} literature
where the energy function $L(\bfw)$ factorizes over small subsets of variables
$\bfw$. 
This leads to efficient minimization of the \acrshort{KL}-divergence as the
expectation $\E_{\bfu}[L(\bfw)]$ can be computed efficiently.
However, in a standard neural network, the function $L(\bfw)$ does not have an
explicit factorization and direct minimization of the \acrshort{KL}-divergence
is not straight forward. 
To simplify the \gls{NN} loss function one can approximate it using its
first-order Taylor approximation which discards the interactions between the
\gls{NN} parameters altogether.

In \thmref{thm:sPGD_mf}, we show that our $\softmax$ based
\gls{PGD} iteratively applies a proximal version of mean-field to the
first-order approximation of $L(\bfw)$.
At iteration $k$, let $\hL^k(\bfw)$ be the first-order Taylor approximation of
$L(\bfw)$.
Then, since there are no interactions among parameters in $\hL^k(\bfw)$, and it
is linear, our proximal mean-field objective has a closed form solution, which
is exactly the $\softmax$ based \gls{PGD} update. 

The following theorem applies to the update of each $\bfu_j\in\Delta$ separately,
and hence the update of the corresponding parameter $w_j$.

\begin{thm}\label{thm:sPGD_mf}
Let $L(\bfu): \Delta \rightarrow \R$ be a differentiable function defined in an
open neighbourhood of the polytope
$\Delta$, and $\bfu^k$ a point in $\Delta$. 
Let $\bfg^k$ be the gradient of $L(\bfu)$ at $\bfu^k$, and
$\hL^k(\bfu) = L(\bfu^k) + \big<\bfu - \bfu^k, \bfg^k \big>$ the
first-order approximation of $L$ at $\bfu^k$. 
Let $\beta$ and $\eta$ (learning rate) be positive constants, and 
\vspace{-2ex}
\begin{equation}
 \bfu^{k+1} = \softmax({\beta(\bfu^k -\eta\,\bfg^k)})\ ,\\[-4ex]
\end{equation}
the $\softmax$-based \gls{PGD} update.  Then,
\vspace{-2ex}
\begin{equation}\label{eq:sPGD_mf_p}
\bfu^{k+1} = \amin{\bfu\in\Delta}\ \eta\,\hL^k(\bfu) - \left\langle\bfu^{k},
\bfu\right\rangle\fro- \frac{1}{\beta}H(\bfu)\ .\\[-1ex]
\end{equation}
\end{thm}
%-----------------------------------------------
%
\begin{proof}
First one shows that
\vspace{-1ex}
\begin{equation}
\eta\,\hL^k(\bfu) - \left\langle\bfu^{k},\bfu\right\rangle\fro = -
\left\langle\bfu, \bfu^k -\eta\,\bfg^k\right\rangle\fro\ ,\\[-1ex]
\end{equation}
 apart from constant
terms (those not containing $\bfu$).
Then the proof follows from~\lemref{lem:sm}.
%See Appendix~\myref{D}.
\end{proof}
\vspace{-1ex}

The objective function~\eqref{eq:sPGD_mf_p} is essentially the same as the mean-field
objective~\plaineqref{eq:mf} for $\hL^k(\bfw)$ (noting $\E_{\bfu}[\hL^k(\bfw)] = \hL^k(\bfu\bfq) = \langle \bfg^k, \bfu \rangle$ up to constant terms) except for the term $\langle
\bfu^k, \bfu\rangle\fro$.
This, in fact, acts as a proximal term. Note, it is the cosine similarity but
subtracted from the loss to enforce proximity. 
Therefore, it encourages the resulting $\bfu^{k+1}$ to be closer to the current
point $\bfu^{k}$ and its influence relative to the loss term is governed by the
learning rate $\eta$. 
Since gradient estimates are stochastic in our case, such a proximal term is
highly desired as it encourages the updates to make a smooth transition.

Furthermore, the negative entropy term acts as a convex regularizer and when
$\beta\to \infty$ its influence becomes negligible and the update results in a
binary labelling ${\bfu\in\calV^m}$.
%It is interesting to note that the $\softmax$ projection in~\eqref{eq:spgdup}
% translates into an entropy term in the objective
% function~\plaineqref{eq:sPGD_mf_p}, and for small values of $\beta$, it allows
% the iterative procedure to explore the optimization landscape.
%Such an explorative behaviour (at least in the beginning of optimization) is
% desirable especially for a stochastic optimization algorithm. 
%
%
% We would like to point out that, in mean-field, the fixed point update is
% performed iteratively until convergence. However, in our case, only one step
% of 
% gradient descent is performed, meaning,~\eqref{eq:sPGD_mf_p} is not minimized
% until convergence.
%More precisely, we show that the \gls{SGD} update followed by the $\softmax$
% projection (refer~\eqref{eq:proj}), 
%From~\thmref{thm:sPGD_mf}, it is clear that the \gls{SGD} update followed by
% the $\softmax$ projection~\plaineqref{eq:spgdup}, which is an approximate
% solution to the~\eqref{eq:taylor}, is actually an exact fixed point solution of
% the proximal mean-field objective~\plaineqref{eq:sPGD_mf_p}.
In addition, the entropy term in~\eqref{eq:sPGD_mf_p} captures the
(in)dependency between the parameters.
To encode dependency, the entropy of the fully-factorized distribution can
perhaps be replaced with a more complex entropy such as a tree-structured
entropy, following the idea of~\cite{ravikumar2008message}. 
Furthermore, in place of $\hL^k$, a higher-order approximation can be used.
Such explorations go beyond the scope of this paper.

\begin{rmk}%\NOT{describe better}
Note that, our update~\plaineqref{eq:sPGD_mf_p} can be interpreted as an {\em entropic penalty method} and it  
is similar in spirit to that of the mirror-descent algorithm when entropy is chosen as the mirror-map (refer Sec.~\myref{4.3}
of~\cite{bubeck2015convex}).
%In fact, compared to mirror-descent, our update constitutes a proximal term and
% an annealing hyperparameter $\beta$ which enables us to gradually enforce a
% discrete solution.
In fact, at each iteration, both our algorithm and mirror-descent augment the
gradient descent objective with a negative entropy term and optimizes over the
polytope.
However, compared to mirror-descent, our update additionally constitutes a
proximal term and an annealing hyperparameter $\beta$ which enables us to
gradually enforce a discrete solution.
Therefore, to employ mirror-descent, one needs to understand the effects of
using adaptive mirror-maps (that depend on $\beta$). % which, we believe, goes
% beyond the scope of this paper.
Nevertheless, it is interesting to explore the potential of mirror-descent which
could allow us to derive different variants of our algorithm.
\end{rmk}

 %%%%%%%%%%%%%%%% PGD-S %%%%%%%%%%%%%
  \begin{algorithm}[t]
\caption{\acrfull{PMF}}
\label{alg:pgds}
\begin{algorithmic}[1]

% \Require Number of iterations $K$, batch size $b$, learning
% rate schedule $\{\eta^k\}$ and growth factor $\rho >1$ 
\Require $K, b, \{\eta^k\}, \rho >1, \calD, \tL$ 
\Ensure $\bfw^*\in\calQ^m$

\State $\tbfu^0\in \R^{m\times d},\quad \beta\gets 1$ 
\Comment{Initialization}

\For{$k \gets 0,\ldots, K$} 

\State $\bfu^k \gets \softmax({\beta\tbfu^k})$
\Comment{Projection (\eqref{eq:spgdup})}

\State $\calD^{b} = \{\left( \bfx_i, \bfy_i \right)\}^{b}_{i=1} \sim
\mathcal{D}$
\Comment{Sample a mini-batch}

\State $\bfg_{\bfu}^k \gets
\left.\nabla_{\bfu}\tL(\bfu;\calD^b)\right|_{\bfu=\bfu^k}$
\Comment{Gradient \wrt $\bfu$ at $\bfu^k$}

\State $\bfg_{\tbfu}^k \gets \bfg_{\bfu}^k \left.\frac{\partial \bfu}{\partial
{\tbfu}}\right|_{\tbfu=\tbfu^k}$
\Comment{Gradient \wrt $\tbfu$ at $\bfu^k$}

\State $\tbfu^{k+1} \gets \tbfu^{k} - \eta^k \bfg_{\tbfu}^k$
\Comment{Gradient descent on $\tbfu$}

\State $\beta\gets \rho\beta$
\Comment{Increase $\beta$}

\EndFor

\State $\bfw^* \gets \hardmax({\tbfu^K})\bfq$
%\\\Return $\bfu^{K}\bfq$
\Comment{Quantization (\eqref{eq:hm})}

\end{algorithmic}
\end{algorithm}

\vspace*{-0.2in}
\paragraph{Proximal Mean-Field (\acrshort{PMF}).}
The preferred embodiment of our \acrshort{PMF} algorithm is similar to softmax based \acrshort{PGD}. 
\algref{alg:pgds} summarizes our approach. 
Similar to the existing methods~\cite{hubara2017quantized}, however, we introduce the 
auxiliary variables $\tbfu\in\R^{m\times d}$ and perform gradient descent on
them, composing the loss function $\tilde L$ with the $\softmax$ function that
maps $\tbfu$ into $\calS$.  
In effect this
 solves the optimization problem: 
 \begin{equation}
  \min_{\tbfu\in\R^{m\times d}}  \tL\left(\softmax (\beta\tbfu);\calD\right)\ .
 \end{equation}
 In this way, optimization is carried out
over the unconstrained domain $\R^{m\times d}$ rather than over the domain
$\calS$.
In contrast to existing methods, this is not a necessity but empirically it
improves the performance.
Finally, since $\beta$ can never be $\infty$, to ensure a fully-quantized
network, the final quantization is performed using $\hardmax$.
%This is equivalent to performing \gls{MAP} estimate on the learned probability
%distribution $\bfu\in\calS$. 
Since, $\softmax$ approaches $\hardmax$ when $\beta\to \infty$, the fixed points
of \algref{alg:pgds} corresponds to the fixed points of \acrshort{PGD} with the
$\hardmax$ projection.
However, exploration due to $\softmax$ allows our algorithm to converge to fixed
points with better validation errors as demonstrated in the  experiments. 

\subsection{Proximal ICM as a Special Case}\label{sec:picm}
For \gls{PGD}, if $\hardmax$ is used instead of the $\softmax$ projection, the
resulting update is the same as a proximal version of
\gls{ICM}~\cite{besag1986statistical}. 
\SKIP{
To see this, let us define the $\hardmax$ projection at iteration $k$ as:
\begin{align}\label{eq:hm}
\bfu^k_j &= \hardmax({\tbfu^k_j})\ ,\\\nonumber
u^k_{j:\lambda} &= \left\{\begin{array}{ll}
1 & \mbox{if $\lambda = \amax{\mu\in\calQ}\, \tu^k_{j:\mu}$}\\
0 & \mbox{otherwise} \end{array} \right. \quad \forall\, \lambda\in\alllabels\ .
\end{align}
where $\tbfu^k = \bfu^{k-1} - \eta\,\bfg^{k-1}$.
}
In fact, following the proof of~\lemref{lem:sm}, it can be shown that the update
$\bfu^{k+1} = \hardmax(\bfu^k - \eta\,\bfg^k)$ yields a fixed point of the
following equation:
\vspace{-1ex}
\begin{equation}\label{eq:hPGD_mf_p}
\min_{\bfu\in\calS}\ \eta\left\langle \bfg^{k}, \bfu\right\rangle\fro -
\left\langle\bfu^{k}, \bfu\right\rangle\fro\ .
\vspace{-1ex}
\end{equation}
Notice, this is exactly the same as the \gls{ICM} objective augmented by the
proximal term.
In this case, $\bfu\in\calV^m\subset\calS$, meaning, the feasible
domain is restricted to be the vertices of the polytope $\calS$.
Since $\softmax$ approaches $\hardmax$ when $\beta\to \infty$, this is a special
case of proximal mean-field.

% %%%%%%%%%%%%%%%% BC %%%%%%%%%%%%%
% \begin{algorithm}[t]
% % \setlength{\textfloatsep}{0pt}
% % \setlength{\floatsep}{0pt}
% \caption{Non-stochastic \acrfull{BC}~\cite{courbariaux2015binaryconnect}}
% \label{alg:bc}
% \begin{algorithmic}[1]
% \Require $K,\eta_{\bfw},\calD,L,\calQ$ 
% \Ensure $\bfw^*\in\calQ^m$
% 
% \State $\tbfw^0\in\R^m$ 
% \Comment{Initialization}
% 
% \For{$k \gets 0\ldots K$} 
% 
% \State $\bfw^k \gets \sign(\tbfw^k)$
% \Comment{Projection}
% 
% %\State $\calD^{b} = \{\left( \bfx_i, \bfy_i \right)\}^{b}_{i=1} \sim
% %\mathcal{D}$
% 
% \State $\bfg_{\bfw}^k \gets
% \left.\nabla_{\bfw}L(\bfw;\calD)\right|_{\bfw=\bfw^k}$
% \Comment{Gradient \wrt $\bfw$}
% 
% \State $\bfg_{\tbfw}^k \gets \bfg_{\bfw}^k \left.\frac{\partial
% \bfw}{\partial
% {\tbfw}}\right|_{\tbfw=\tbfw^k}$
% \Comment{Gradient \wrt $\tbfw$}
% 
% \State $\tbfw^{k+1} \gets \tbfw^{k} - \eta_{\bfw}\,\bfg_{\tbfw}^k$
% \Comment{Gradient descent}
% 
% \EndFor
% 
% \State $\bfw^* \gets \sign(\tbfw^{K})$
% \Comment{Final quantization}
% %\\\Return $\bfw^{K}$
% 
% \end{algorithmic}
% \end{algorithm}

\subsection{BinaryConnect as Proximal ICM}\label{sec:bc}
In this section, considering binary neural networks, \ie, $\calQ=\{-1,1\}$, and
non-stochastic setting, we show that the \gls{PICM} algorithm is
equivalent to the popular \gls{BC}
method~\cite{courbariaux2015binaryconnect}. 
In these algorithms, the gradients are computed in two different spaces and
therefore to alleviate any discrepancy we assume that gradients are computed
using the full dataset.
%, \ie, the non-stochastic setting.

% For the scope of this section, to alleviate any discrepancey due to the
% stochasticity, we assume
% that the non-stochastic version of gradient descent is employed in both the
% algorithms.\NOTE{may not be necessary}

Let $\tbfw\in\R^m$ and $\bfw\in\calQ^m$ be the infeasible and feasible
points of \gls{BC}. Similarly, $\tbfu\in\R^{m\times d}$ and
$\bfu\in\calV^m\subset\calS$ be the infeasible and feasible points of our 
\gls{PICM} method, respectively. 
% Note that, in BinaryConnect, the projection to the
% feasible domain is performed by the $\sign$ function, \ie, $\bfw =
% \sign(\tbfw)$. Similarly, it is $\bfu=\hardmax(\tbfu)$ in proximal ICM.
%
%For convenience, we summarize binary connect~(\algref{alg:bc}) as a projected
% gradient descent method in $\bfw$-space.
For convenience, we summarize one iteration of \gls{BC} in~\algref{alg:bc}.
%Recall that, \gls{PICM} is the same as~\algref{alg:pgds} except, in
% \gls{PICM},
%the $\softmax$ projection is replaced by $\hardmax$.
%
\SKIP{
Let us first recall that the variables $\bfw$ and $\bfu$ are related as
(\eqref{eq:indicv}):
\begin{align}
\tbfw &= \tbfu\bfq\ ,\\\nonumber
 \bfw &= \bfu\bfq\ .
\end{align}
} 
%Let us first recall that the variables $\bfw$ and $\bfu$ are related as
% $\bfw=\bfu\bfq$ (\eqref{eq:indicv}).
Now, we show that the update steps in both \gls{BC} and \gls{PICM} are
equivalent.

%%%%%%%%%%%%%%%% BC %%%%%%%%%%%%%
\begin{algorithm}[t]
\caption{One iteration of \acrfull{BC}~\cite{courbariaux2015binaryconnect}}
\label{alg:bc}
\begin{algorithmic}[1]
\Require $\tbfw^k,\eta_{\bfw},\calD,L$ 
%\Ensure $\bfw^*\in\calQ^m$

%\State $\tbfw^0\in\R^m$ 
%\Comment{Initialization}

%\For{$k \gets 0\ldots K$} 

\State $\bfw^k \gets \sign(\tbfw^k)$
\Comment{Projection}

%\State $\calD^{b} = \{\left( \bfx_i, \bfy_i \right)\}^{b}_{i=1} \sim
%\mathcal{D}$

\State $\bfg_{\bfw}^k \gets
\left.\nabla_{\bfw}L(\bfw;\calD)\right|_{\bfw=\bfw^k}$
\Comment{Gradient \wrt $\bfw$}

\State $\bfg_{\tbfw}^k \gets \bfg_{\bfw}^k \left.\frac{\partial \bfw}{\partial
{\tbfw}}\right|_{\tbfw=\tbfw^k}$
\Comment{Gradient \wrt $\tbfw$}

\State $\tbfw^{k+1} \gets \tbfw^{k} - \eta_{\bfw}\,\bfg_{\tbfw}^k$
\Comment{Gradient descent}

%\EndFor

%\State $\bfw^* \gets \sign(\tbfw^{K})$
%\Comment{Final quantization}
%\\\Return $\bfw^{K}$

\end{algorithmic}
\end{algorithm}

\begin{pro}\label{pro:bc_vs_pgdh}
Consider \gls{BC} and \gls{PICM} with $\bfq = [-1,1]^T$ and $\eta_{\bfw} > 0$.
For an iteration $k>0$, if $\tbfw^k= \tbfu^k\bfq$ then, 
\begin{enumerate}
  \item the projections in \acrshort{BC}: ${\bfw^k = \sign(\tbfw^k)}$ and\\ 
   \acrshort{PICM}: ${\bfu^k = \hardmax(\tbfu^k)}$  
  satisfy  ${\bfw^k = \bfu^k\bfq}$.
  
  \item let the learning rate of \acrshort{PICM} be $\eta_{\bfu} = \eta_{\bfw}/2$, then the updated points after the
gradient descent step in \acrshort{BC} and \acrshort{PICM} satisfy
  ${\tbfw^{k+1} = \tbfu^{k+1}\bfq}$.
\end{enumerate} 
\end{pro}
\begin{proof} 
Case (1) is simply applying $\tbfw^k=\tbfu^k\tbfq$ whereas case (2) can be
proved by writing $\bfw^k$ as a function of $\tbfu^k$ and then applying chain
rule. 
See Appendix~\myref{B}.
\end{proof}

Since $\hardmax$ is a non-differentiable operation, the partial derivative 
$\partial \bfu/\partial \tbfu = \partial \hardmax/\partial \tbfu$ is not
defined.
However, to allow backpropagation, we write $\hardmax$ in terms of the $\sign$
function, and used the 
straight-through-estimator~\cite{stehinton} to allow gradient flow similar to
binary connect.
For details please refer to Appendix~\myref{B.1}.

\SKIP{
Furthermore, $\hardmax$ is a non-differentiable operation but $\partial
\bfu/\partial \tbfu = \partial \hardmax/\partial \tbfu$ needs to be computed.
Interestingly, the $\hardmax$ function can be written in terms of the $\sign$
function
and then the partial derivative can be approximated using the
straight-through-estimator~\cite{stehinton} similar to binary connect.
For details of the derivation please refer to Appendix~\myref{E.1}.
}

\section{Related Work}
There is much work on \gls{NN} quantization focusing on different aspects such
as quantizing parameters~\cite{courbariaux2015binaryconnect},
activations~\cite{hubara2016binarized}, loss aware
quantization~\cite{hou2016loss} and quantization for specialized
hardware~\cite{esser2015energyEfficient}, to name a few.
Here we give a brief summary of latest works and for a comprehensive survey we
refer the reader to~\cite{guo2018survey}.

In this work, we consider parameter quantization, which can either be treated as
a post-processing scheme~\cite{gong2014vectorQuantization} or incorporated into
the learning process.
Popular methods~\cite{courbariaux2015binaryconnect,hubara2017quantized} falls
into the latter category and optimize the constrained problem using some form of
projected stochastic gradient descent.
In contrast to projection, quantization can also be enforced using a penalty
term~\cite{bai2018proxquant,yin2018binaryrelax}.
Even though, our method is based on projected gradient descent, by optimizing in
the $\bfu$-space, we provide theoretical insights based on mean-field and bridge
the gap between \gls{NN} quantization and \gls{MRF} optimization literature.

In contrast, the variational approach can also be used for quantization, where
the idea is to learn a posterior probability on the network parameters in a
Bayesian framework.  
In this family of methods, the quantized network can be obtained either via a
quantizing prior~\cite{achterhold2018variationalQuantization} or using the
\acrshort{MAP} estimate on the learned posterior~\cite{soudry2014EBP}.
Interestingly, the learned posterior distribution can be used to estimate the
model uncertainty and in turn determine the required precision for each network
parameter~\cite{louizon2017bayesianCompression}.
Note that, even in our seemingly different method, we learn a probability
distribution over the parameters (see~\secref{sec:uspace}) and it would be
interesting to understand the connection between Bayesian methods and our
algorithm.

\section{Experiments}

%As mentioned previously, we restrict our attention to quantizing the parameters
% of the neural network and consider the extreme case where all learnable
% parameters are quantized.
Since neural network binarization is the most popular 
quantization~\cite{courbariaux2015binaryconnect,rastegari2016xnor},
we set the quantization levels to be binary, \ie, $\calQ=\{-1,1\}$.
However, our formulation is applicable to any predefined set of quantization
levels given
sufficient resources at training time.
We would like to point out that, we quantize all learnable parameters, meaning,
all quantization algorithms result in $32$ times less memory compared to the
floating point counterparts.

We evaluate our \gls{PMF} algorithm on \mnist{}, \cifar{-10},
\cifar{-100} and
\tinyimagenet{}\footnote{\url{https://tiny-imagenet.herokuapp.com/}} 
classification datasets with convolutional and residual architectures and
compare against the \gls{BC} method~\cite{courbariaux2015binaryconnect} and the
latest algorithm \gls{PQ}~\cite{bai2018proxquant}. 
Note that \gls{BC} and \gls{PQ} constitute the closest and directly comparable
baselines to \gls{PMF}. 
Furthermore, many other methods have been developed based on \gls{BC} by
relaxing some of the constraints, \eg, layer-wise
scalars~\cite{rastegari2016xnor}, and we believe, similar extensions are
possible with our method as well.
Our results show that the binary networks obtained by \gls{PMF} yield accuracies
very close to the floating point counterparts while consistently outperforming
the baselines.

%%%%%%%%%%%% setup
\begin{table}[t]
    \centering
    %\small    
    \begin{tabular}{l@{\hspace{.7\tabcolsep}}c@{\hspace{.7\tabcolsep}}c@{
\hspace{.7\tabcolsep}}c@{\hspace{.7\tabcolsep}}c@{\hspace{.7\tabcolsep}}c@{
\hspace{.7\tabcolsep}}c@{\hspace{.7\tabcolsep}}}
        \toprule
        Dataset & Image & \# class & Train~/~Val. & $b$ & $K$ \\
        \midrule
        \mnist & $28\times 28$ & $10$ & $50$k~/~$10$k  & $100$ & $20$k\\
        \cifar{-10} & $32\times 32$ & $10$ & $45$k~/~$5$k  & $128$ & $100$k\\
        \cifar{-100} & $32\times 32$ & $100$ & $45$k~/~$5$k  & $128$ & $100$k\\
        \tinyimagenet{} & $64\times 64$ & $200$ & $100$k~/~$10$k  & $128$ &
$100$k\\
        \bottomrule
    \end{tabular}
    \vspace{1ex}
    \caption{\em
    	Experiment setup. Here, $b$ is the batch size and $K$ is the
    	total number of iterations used for all the methods.}
    \label{tab:setup}
\end{table}

%%%%%%% results
\begin{table*}[t]
    \centering
    %\small    
    \begin{tabular}{llccc|cccc}
        \toprule
        %Dataset & Architecture & \acrshort{REF} {\small (Float)} & 
        %\acrshort{BC}~\cite{courbariaux2015binaryconnect} &
% \acrshort{PQ}~\cite{bai2018proxquant} & \acrshort{PICM} & \acrshort{PGD} &
% \acrshort{PMF} & \acrshort{REF} - \acrshort{PMF}\\
        \multirow{2}{*}{Dataset} & \multirow{2}{*}{Architecture} &
\multirow{2}{*}{\acrshort{REF} {\small (Float)}}  & 
\multirow{2}{*}{\acrshort{BC}~\cite{courbariaux2015binaryconnect}}
         &  \multirow{2}{*}{\acrshort{PQ}~\cite{bai2018proxquant}} &
\multicolumn{3}{c}{Ours} &  \multirow{2}{*}{\acrshort{REF} - \acrshort{PMF}}\\
         &  &  &  &  & \acrshort{PICM} & \acrshort{PGD} & \acrshort{PMF} & \\
        \midrule
        \multirow{2}{*}{\mnist}
         & \slenet{-300} & $98.55$ & $98.05$ & $98.13$ & $98.18$ & $98.21$ &
$\textbf{98.24}$ & $+0.31$\\
		 & \slenet{-5}   & $99.39$ & $99.30$ & $99.27$ & $99.31$ & $99.28$ &
$\textbf{99.44}$ & $-0.05$\\
        \midrule
        \multirow{2}{*}{\cifar{-10}}
         & \svgg{-16} 	 & $93.01$ & $86.40$ & $90.11$ & $88.96$ & $88.48$ &
$\textbf{90.51}$ & $+2.50$\\
		 & \sresnet{-18} & $94.64$ & $91.60$ & $92.32$ & $92.02$ & $92.60$ &
$\textbf{92.73}$ & $+1.91$\\
         \midrule
        \multirow{2}{*}{\cifar{-100}}
         & \svgg{-16} 	 & $70.33$ & $43.70$ & $55.10$ & $45.65$ & $57.83$ &
$\textbf{61.52}$ & $+8.81$\\
		 & \sresnet{-18} & $73.85$ & $69.93$ & $68.35$ & $70.85$ & $70.60$ &
$\textbf{71.85}$ & $+2.00$\\
         \midrule
         \tinyimagenet
         & \sresnet{-18} & $56.41$ & $49.33$ & $49.97$ & $49.66$ & $49.60$ &
$\textbf{51.00}$ & $+5.63$\\
        \bottomrule
    \end{tabular}
    %\vspace{-1ex}
    \caption{\em
    	Classification accuracies on the test set for different methods. Note that
our \acrshort{PMF}
    	algorithm consistently produces better results than other binarization
    	methods and the degradation in performance to the full floating
    	point network (last column) is minimal especially for small datasets. For
larger datasets
    	(\eg, \cifar{-100}), binarizing \sresnet{-18} results in much smaller
degradation
    	compared to \svgg{-16}. Even though, \acrshort{PICM} and \acrshort{BC} are
    	theoretically equivalent in the non-stochastic setting, \acrshort{PICM}
yields
    	slightly better accuracies. Note, all binarization methods except
\acrshort{PQ} require exactly $\mathbf{32}$ times less memory compared to
single-precision floating points networks at test time.
    	%\NOTE{incomplete columns are marked *}
    	}
        \vspace{-1ex}
    \label{tab:res}
\end{table*}

\subsection{Experiment Setup}
The details of the datasets and their corresponding experiment setups are given
in~\tabref{tab:setup}. 
In all the experiments, standard multi-class cross-entropy loss is minimized.
\mnist{} is tested using
\slenet{-300} and \slenet{-5}, where the former consists of three \gls{fc}
layers while the latter is composed of two convolutional and two \gls{fc}
layers.
For \cifar{} and \tinyimagenet{}, \svgg{-16}~\cite{simonyan2014very} and
\sresnet{-18}~\cite{he2016deep} architectures adapted for \cifar{} dataset
are used. 
In particular, for \cifar{} experiments, similar to~\cite{lee2018snip}, the
size of the \gls{fc} layers of \svgg{-16} is set to $512$ and no dropout layers
are employed.
For \tinyimagenet{}, the stride of the first convolutional layer of
\sresnet{-18} is set to $2$ to handle the
image size~\cite{huang2017snapshot}.
In all the models, batch normalization~\cite{ioffe2015batch} (with no
learnable parameters) and \srelu{} non-linearity are used.
 Except for \mnist, standard data augmentation is used (\ie, random crop
and horizontal flip) and weight decay is set to $0.0001$ unless stated
otherwise.

For all the algorithms, the hyperparameters such as the optimizer and
the learning rate (also its schedule) are cross-validated using the validation
set\footnote{For \tinyimagenet{}, since the ground truth labels for
the test set were not available, validation set is used for both
cross-validation and testing.} and the chosen parameters are given in the
supplementary material. 
For \acrshort{PMF} and \acrshort{PGD} with $\sparsemax$, the growth-rate $\rho$
in~\algref{alg:pgds} (the multiplicative factor used to increase $\beta$) is
cross validated between $1.01$ and $1.2$ and chosen values for each experiment
are given in supplementary. 
%The growth-rate $\rho$ in~\algref{alg:pgds} (the multiplicative factor used to
% increase $\beta$) is set to $1.05$ ($1.01$ for \acrshort{PGD} with $\sparsemax$)
% except for \mnist{} (where it is $1.2$) and $\beta$ is increased (line $8$
% in~\algref{alg:pgds}) every $100$ iterations.
Furthermore, since the original implementation of \gls{BC} do not binarize all
the learnable parameters, for fair comparison, we implemented \gls{BC} in our
experiment setting based on the publicly available
code\footnote{\url{https://github.com/itayhubara/BinaryNet.pytorch}}.
However, for \gls{PQ} we used the original
code\footnote{\url{https://github.com/allenbai01/ProxQuant}}, \ie, {\em for
\gls{PQ}, biases and last layer parameters are not binarized}. 
All methods are trained from a random initialization and the model with the best
validation accuracy is chosen for each method.
%Note that, in \gls{PMF}, even though we use an increasing schedule for $\beta$
% to enforce a discrete solution, the chosen network may not be fully-quantized
% (as the best model could be obtained in an early stage of training).
%Therefore, the $\hardmax$ projection is applied to ensure that the network is
% fully-quantized.
Our algorithm is implemented in PyTorch~\cite{paszke2017automatic}. % and the
% code is available at \href{https://github.com/tajanthan/pmf}
% {https://github.com/tajanthan/pmf}.

%%%%%%% results
\begin{figure*}[t]
    \centering
    \begin{subfigure}{0.25\linewidth}
    \includegraphics[width=0.99\linewidth]
{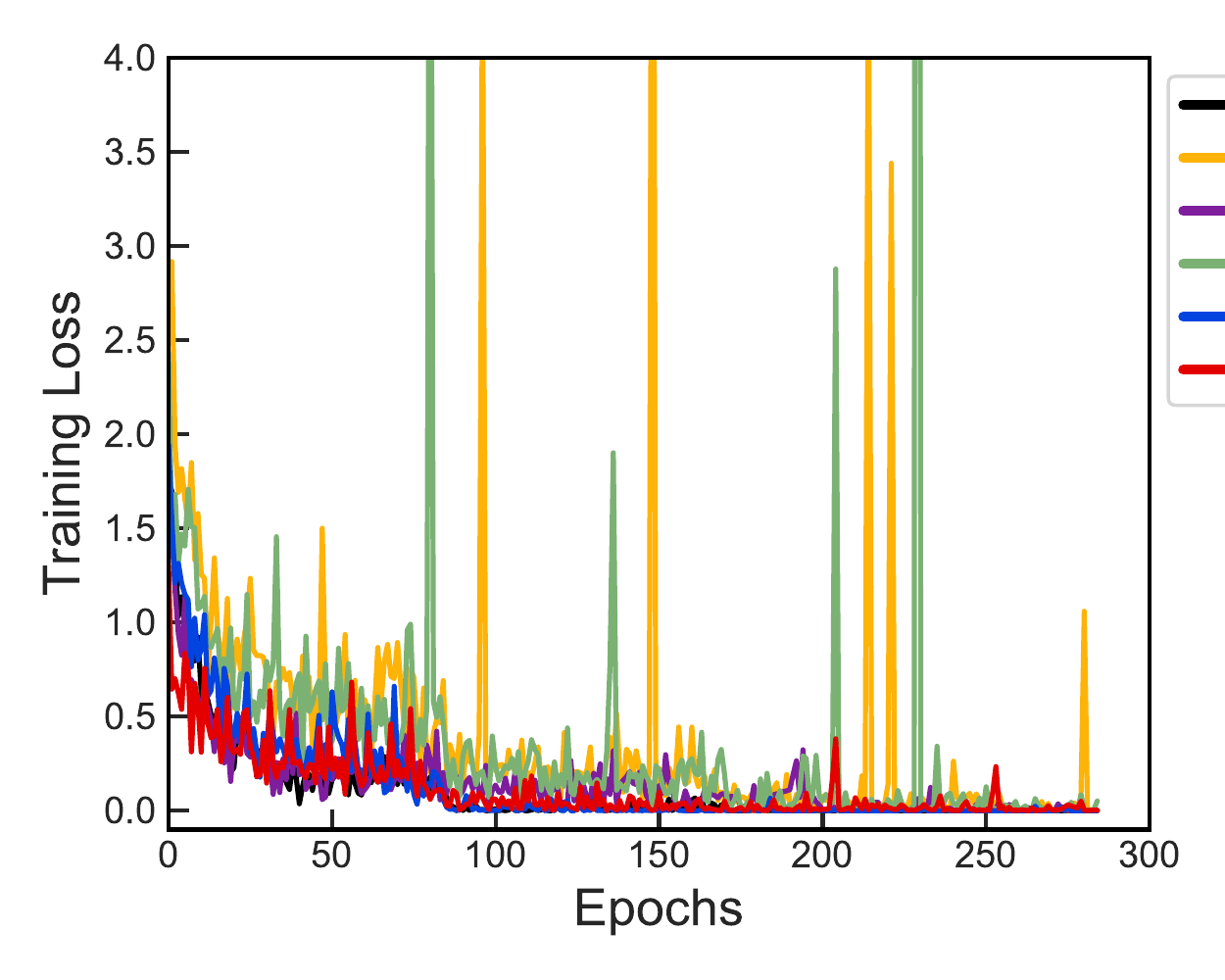}
    \end{subfigure}%
    \begin{subfigure}{0.25\linewidth}
    \includegraphics[width=0.99\linewidth]
{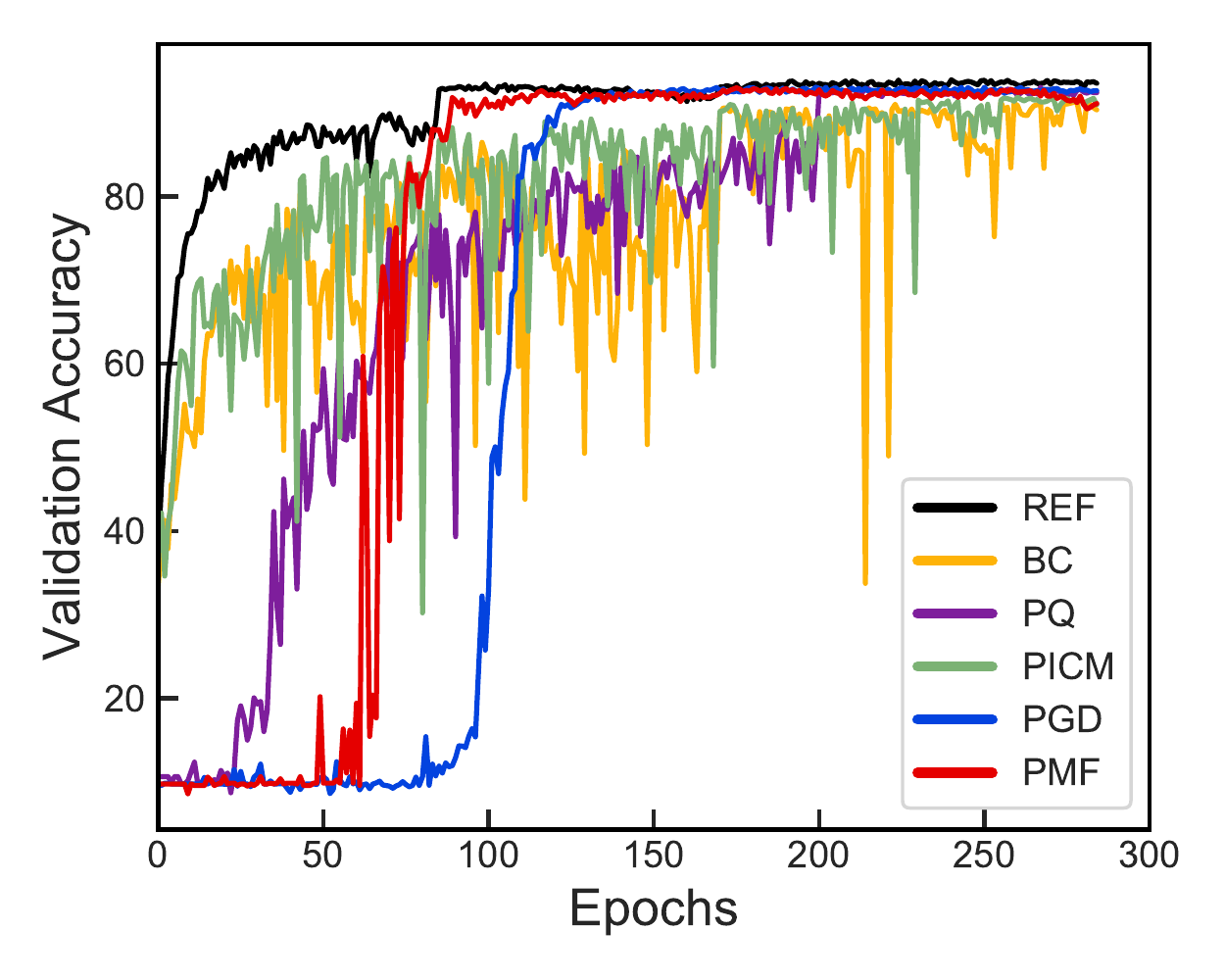}
    \end{subfigure}%
    \begin{subfigure}{0.25\linewidth}
    \includegraphics[width=0.99\linewidth]
{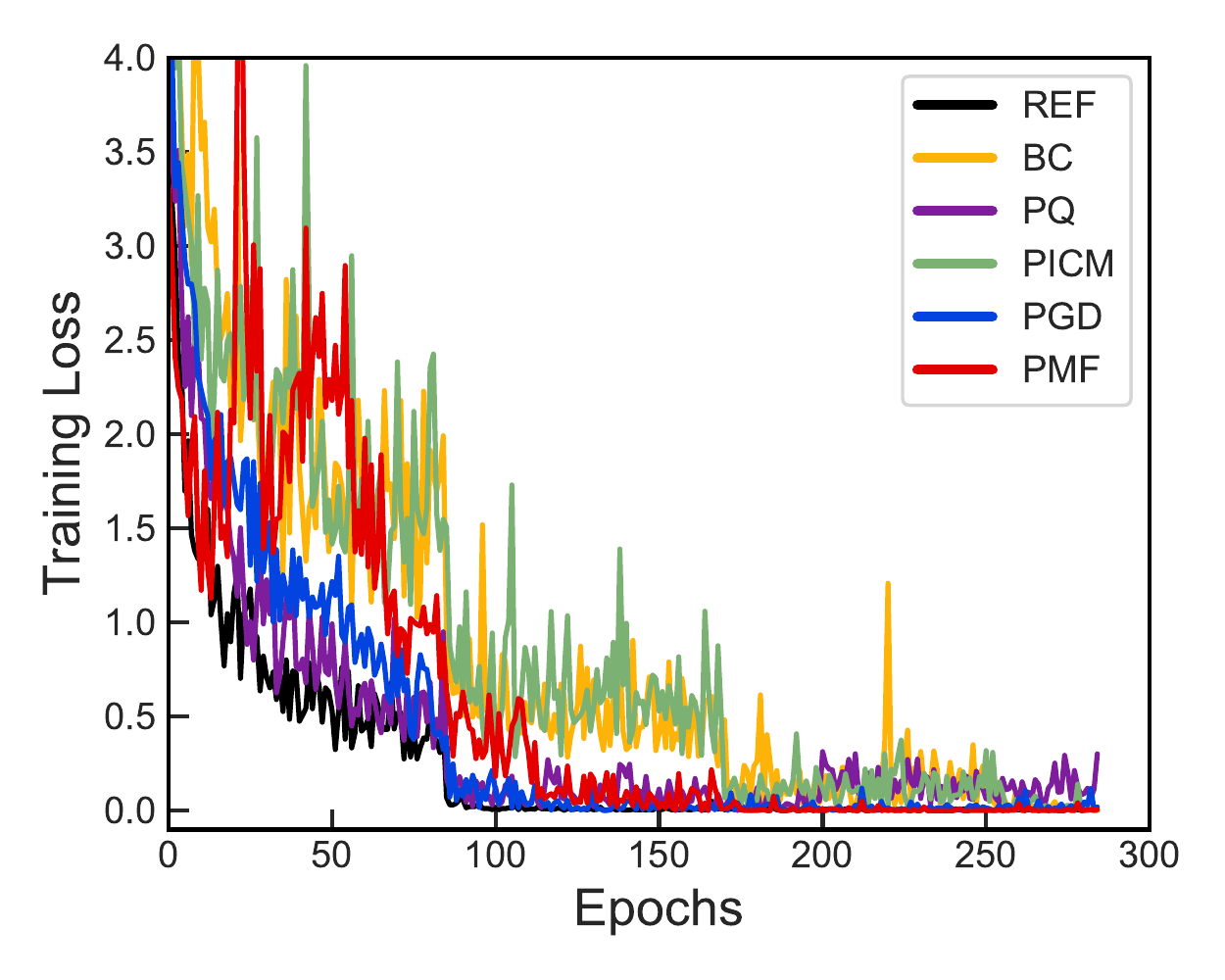}
    \end{subfigure}%
    \begin{subfigure}{0.25\linewidth}
    \includegraphics[width=0.99\linewidth]
{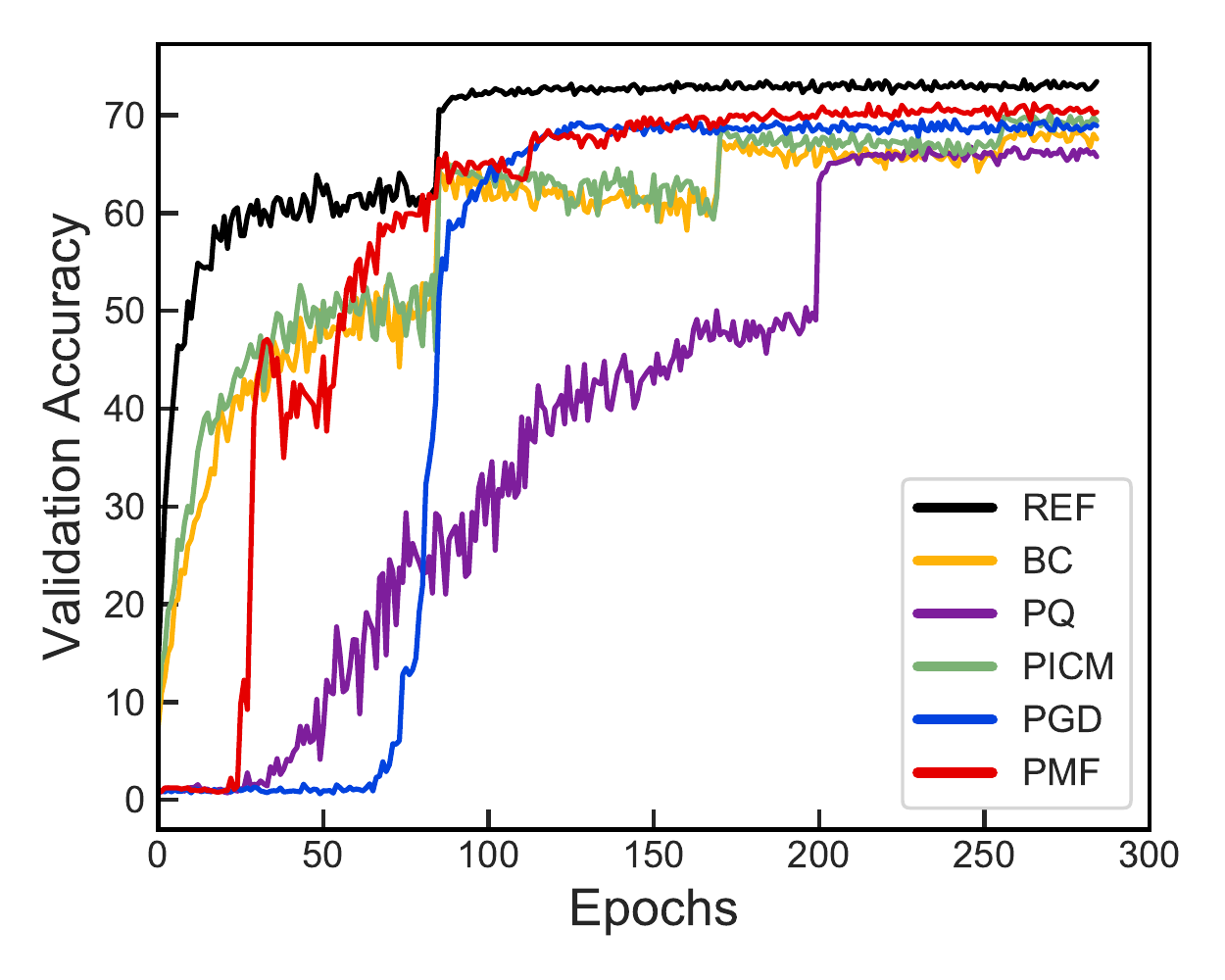}
    \end{subfigure}
    \vspace{-2ex}
    \caption{\em Training curves for \cifar{-10} (first two) and \cifar{-100}
(last two) with \sresnet{-18}. 
    For quantization methods, the validation accuracy is always measured with
the quantized networks. 
    Specifically, for \gls{PMF} and \gls{PGD}, the $\hardmax$ projection is
applied before the evaluation.
    %The validation accuracies of \acrshort{PMF} are the worst at the beginning
% (until $25$ epochs).
    Notably, validation accuracy plots clearly illustrate the exploration phase
of both \acrshort{PMF} and \acrshort{PGD}, during which the accuracies are the
worst.
    However, once $\beta$ is ``large enough'', the curves closely resembles
high-precision reference network while yielding very high accuracies.
    Furthermore, compared to \acrshort{BC} and \acrshort{PICM}, other methods
are less noisy suggesting the usefulness of optimizing over a convex domain. 
    %We believe, this is the exploration phase of \acrshort{PMF} and $\hardmax$
% projection introduces errors in this phase. 
    %Due to errors introduced by the $\hardmax$ projection, at the beginning
% (until $25$ epochs), the validation accuracy of \gls{PMF} is the worst.
    %While \acrshort{BC} and \acrshort{PICM} are extremely noisy, \acrshort{PMF}
% training curves are fairly smooth and closely resembles the high-precision
% reference network (especially on \cifar{-10}).
    %Note that, in \cifar{-10}, \acrshort{PMF} surpasses \acrshort{REF} for
% epochs between $30$ -- $80$ and closely follows it afterwards.\NOTE{plots w to
% be updated with PQ and PGD and the caption will be modified} 
    }
    \vspace{-2ex}
    \label{fig:curves}
\end{figure*}

\subsection{Results}
The classification accuracies (top-1) on the test set of all versions of our
algorithm, namely, \gls{PMF}, \gls{PGD} (this is \acrshort{PGD} with the
$\sparsemax$ projection), and \gls{PICM}, the baselines \gls{BC} and \gls{PQ},
and the floating point \gls{REF} are reported in~\tabref{tab:res}. 
The training curves for \cifar{-10} and \cifar{-100} with \sresnet{-18} are
shown in~\figref{fig:curves}. 
Note that our \gls{PMF} algorithm consistently produces better results than
other binarization methods and the degradation in performance to the full
floating point reference network is minimal especially for small datasets. 
For larger datasets (\eg, \cifar{-100}), binarizing \sresnet{-18} results in
much smaller degradation compared to \svgg{-16}.

The superior performance of \gls{PMF} against \gls{BC}, \gls{PICM} and \gls{PGD}
empirically validates the hypothesis that performing ``noisy'' projection via
$\softmax$ and annealing the noise is indeed beneficial in the stochastic
setting. 
Furthermore, even though \gls{PICM} and \gls{BC} are theoretically equivalent in
the non-stochastic setting, \gls{PICM} yields slightly better accuracies in all
our experiments. 
We conjecture that this is due to the fact that in \gls{PICM}, the training is
performed on a larger network (\ie, in the $\bfu$-space).
%Similar behaviour is empirically observed in~\cite{mishra2017wrpn}.

To further consolidate our implementation of \gls{BC}, we quote the accuracies
reported in the original papers here. 
In~\cite{courbariaux2015binaryconnect}, the top-1 accuracy on \cifar{-10} with a
modified \svgg{} type network is $90.10\%$. 
In the same setting, even with additional layer-wise scalars,
(\gls{BWN}~\cite{rastegari2016xnor}), the corresponding accuracy is $90.12\%$.
For comprehensive results on network quantization we refer the reader to
Table~\myref{5} of~\cite{guo2018survey}.
Note that, in all the above cases, the last layer parameters and biases in all
layers were not binarized. 

%\pagebreak
\section{Discussion}
%\vspace{-1.5ex}
In this work, we have formulated \gls{NN} quantization as a discrete labelling
problem and introduced a projected stochastic gradient descent algorithm to
optimize it.
By showing our approach as a proximal mean-field method, we have also provided
an \gls{MRF} optimization perspective to \gls{NN} quantization.
This connection opens up interesting research directions primarily on
considering dependency between the neural network parameters to derive better
network quantization schemes.
Furthermore, our \gls{PMF} approach learns a probability distribution over the
network parameters, which is similar in spirit to Bayesian deep learning
methods.
Therefore, we believe, it is interesting to explore the connection between
Bayesian methods and our algorithm, which can potentially drive research in both
fields.

%Furthermore, in our approach, entropy acts as a regularizer and as shown
% empirically, having a regularizer and annealing its influence results in
% improved performance in the stochastic setting.
%Therefore, we believe, this would also encourage research on better
% regularization methods focusing on \gls{NN} quantization. 

\section{Acknowledgements}
This work was supported by the ERC grant ERC-2012-AdG 321162-HELIOS, EPSRC grant
Seebibyte EP/M013774/1, EPSRC/MURI grant EP/N019474/1 and the Australian
Research Council Centre of Excellence for Robotic Vision (project number
CE140100016). 
We would also like to acknowledge the Royal Academy of Engineering, FiveAI,
National Computing Infrastructure, Australia and Nvidia (for GPU donation).

\ifarxiv
\appendices

\fi
\fi

{\small
\bibliographystyle{ieee_fullname}
\bibliography{pmf}

\begin{thebibliography}{10}\itemsep=-1pt

\bibitem{achterhold2018variationalQuantization}
J. Achterhold, J.~M. Kohler, A. Schmeink, and T. Genewein.
\newblock Variational network quantization.
\newblock {\em ICLR}, 2018.

\bibitem{ajanthanphdthesis}
Thalaiyasingam Ajanthan.
\newblock {\em Optimization of {M}arkov random fields in computer vision}.
\newblock PhD thesis, Australian National University, 2017.

\bibitem{ajanthan2016efficient}
Thalaiyasingam Ajanthan, Alban Desmaison, Rudy Bunel, Mathieu Salzmann, Philip
  H~S Torr, and M~Pawan Kumar.
\newblock Efficient linear programming for dense {CRF}s.
\newblock {\em CVPR}, 2017.

\bibitem{bai2018proxquant}
Yu Bai, Yu-Xiang Wang, and Edo Liberty.
\newblock Proxquant: Quantized neural networks via proximal operators.
\newblock {\em ICLR}, 2019.

\bibitem{besag1986statistical}
Julian Besag.
\newblock On the statistical analysis of dirty pictures.
\newblock {\em Journal of the Royal Statistical Society.}, 1986.

\bibitem{blake2011markov}
Andrew Blake, Pushmeet Kohli, and Carsten Rother.
\newblock {\em Markov random fields for vision and image processing}.
\newblock Mit Press, 2011.

\bibitem{boyd2009convex}
Stephen Boyd and Lieven Vandenberghe.
\newblock {\em Convex optimization}.
\newblock Cambridge university press, 2009.

\bibitem{bubeck2015convex}
S{\'e}bastien Bubeck.
\newblock Convex optimization: Algorithms and complexity.
\newblock {\em Foundations and Trends{\textregistered} in Machine Learning},
  2015.

\bibitem{chekuri2004linear}
Chandra Chekuri, Sanjeev Khanna, Joseph Naor, and Leonid Zosin.
\newblock A linear programming formulation and approximation algorithms for the
  metric labeling problem.
\newblock {\em SIAM Journal on Discrete Mathematics}, 2004.

\bibitem{courbariaux2015binaryconnect}
Matthieu Courbariaux, Yoshua Bengio, and Jean-Pierre David.
\newblock Binaryconnect: Training deep neural networks with binary weights
  during propagations.
\newblock {\em NIPS}, 2015.

\bibitem{dokania2015parsimonious}
P.~K. Dokania and P.~K. Mudigonda.
\newblock Parsimonious labeling.
\newblock {\em ICCV}, 2015.

\bibitem{esser2015energyEfficient}
S.~K. Esser, R. Appuswamy, P.~A. Merolla, J.~V. Arthur, and D.~S. Modha.
\newblock Backpropagation for energy-efficient neuromorphic computing.
\newblock {\em NIPS}, 2015.

\bibitem{frank1956algorithm}
Marguerite Frank and Philip Wolfe.
\newblock An algorithm for quadratic programming.
\newblock {\em Naval research logistics quarterly}, 1956.

\bibitem{gong2014vectorQuantization}
Y. Gong, L. Liu, and L. Bourdev.
\newblock Compressing deep convolutional networks using vector quantization.
\newblock {\em arXiv preprint arXiv:1412.6115}, 2014.

\bibitem{guo2018survey}
Yunhui Guo.
\newblock A survey on methods and theories of quantized neural networks.
\newblock {\em arXiv preprint arXiv:1808.04752}, 2018.

\bibitem{he2016deep}
Kaiming He, Xiangyu Zhang, Shaoqing Ren, and Jian Sun.
\newblock Deep residual learning for image recognition.
\newblock {\em CVPR}, 2016.

\bibitem{stehinton}
Geoffrey Hinton.
\newblock Neural networks for machine learning.
\newblock {\em Coursera, video lectures}, 2012.

\bibitem{hou2016loss}
Lu Hou, Quanming Yao, and James~T Kwok.
\newblock Loss-aware binarization of deep networks.
\newblock {\em ICLR}, 2017.

\bibitem{huang2017snapshot}
Gao Huang, Yixuan Li, Geoff Pleiss, Zhuang Liu, John~E Hopcroft, and Kilian~Q
  Weinberger.
\newblock Snapshot ensembles: Train 1, get m for free.
\newblock {\em ICLR}, 2017.

\bibitem{hubara2016binarized}
Itay Hubara, Matthieu Courbariaux, Daniel Soudry, Ran El-Yaniv, and Yoshua
  Bengio.
\newblock Binarized neural networks.
\newblock {\em NIPS}, 2016.

\bibitem{hubara2017quantized}
Itay Hubara, Matthieu Courbariaux, Daniel Soudry, Ran El-Yaniv, and Yoshua
  Bengio.
\newblock Quantized neural networks: Training neural networks with low
  precision weights and activations.
\newblock {\em JMLR}, 2017.

\bibitem{ioffe2015batch}
Sergey Ioffe and Christian Szegedy.
\newblock Batch normalization: Accelerating deep network training by reducing
  internal covariate shift.
\newblock {\em ICML}, 2015.

\bibitem{mrf1980kindermann}
R. Kindermann and J.~L. Snell.
\newblock {\em Markov Random Fields and Their Applications}.
\newblock American Mathematical Society, 1980.

\bibitem{kingma2014adam}
Diederik~P Kingma and Jimmy Ba.
\newblock Adam: A method for stochastic optimization.
\newblock {\em ICLR}, 2015.

\bibitem{kleinberg2002approximation}
Jon Kleinberg and Eva Tardos.
\newblock Approximation algorithms for classification problems with pairwise
  relationships: metric labeling and {M}arkov random fields.
\newblock {\em Journal of the ACM}, 2002.

\bibitem{kolmogorov2004energy}
Vladimir Kolmogorov and Ramin Zabin.
\newblock What energy functions can be minimized via graph cuts?
\newblock {\em PAMI}, 2004.

\bibitem{lacoste2012block}
Simon Lacoste-Julien, Martin Jaggi, Mark Schmidt, and Patrick Pletscher.
\newblock Block-coordinate {F}rank-{W}olfe optimization for structural {SVM}s.
\newblock {\em ICML}, 2012.

\bibitem{lee2018snip}
Namhoon Lee, Thalaiyasingam Ajanthan, and Philip H~S Torr.
\newblock {SNIP}: Single-shot network pruning based on connection sensitivity.
\newblock {\em ICLR}, 2019.

\bibitem{louizon2017bayesianCompression}
C. Louizos, K. Ullrich, and M. Welling.
\newblock Bayesian compression for deep learning.
\newblock {\em NIPS}, 2017.

\bibitem{martins2016sparsemax}
Andre Martins and Ramon Astudillo.
\newblock From softmax to sparsemax: A sparse model of attention and
  multi-label classification.
\newblock {\em ICML}, 2016.

\bibitem{mudigonda2008combinatorial}
Pawan~Kumar Mudigonda.
\newblock {\em Combinatorial and convex optimization for probabilistic models
  in computer vision}.
\newblock PhD thesis, Oxford Brookes University, 2008.

\bibitem{nemhauser1988integer}
George~L Nemhauser and Laurence~A Wolsey.
\newblock {\em Integer programming and combinatorial optimization}.
\newblock Springer, 1988.

\bibitem{nocedal2006numerical}
Jorge Nocedal and Stephen Wright.
\newblock Numerical optimization.
\newblock {\em Springer}, 2006.

\bibitem{parikh2014proximal}
Neal Parikh and Stephen~P Boyd.
\newblock Proximal algorithms.
\newblock {\em Foundations and Trends in Optimization}, 2014.

\bibitem{paszke2017automatic}
Adam Paszke, Sam Gross, Soumith Chintala, Gregory Chanan, Edward Yang, Zachary
  DeVito, Zeming Lin, Alban Desmaison, Luca Antiga, and Adam Lerer.
\newblock Automatic differentiation in {P}y{T}orch.
\newblock 2017.

\bibitem{rastegari2016xnor}
Mohammad Rastegari, Vicente Ordonez, Joseph Redmon, and Ali Farhadi.
\newblock Xnor-net: Imagenet classification using binary convolutional neural
  networks.
\newblock {\em ECCV}, 2016.

\bibitem{ravikumar2008message}
Pradeep Ravikumar, Alekh Agarwal, and Martin~J Wainwright.
\newblock Message-passing for graph-structured linear programs: proximal
  projections, convergence and rounding schemes.
\newblock {\em ICML}, 2008.

\bibitem{robbins1985stochastic}
Herbert Robbins and Sutton Monro.
\newblock A stochastic approximation method.
\newblock {\em Annals of Mathematical Statistics}, 1951.

\bibitem{rosasco2014convergence}
Lorenzo Rosasco, Silvia Villa, and Bang~C{\^o}ng V{\~u}.
\newblock Convergence of stochastic proximal gradient algorithm.
\newblock {\em arXiv preprint arXiv:1403.5074}, 2014.

\bibitem{simonyan2014very}
Karen Simonyan and Andrew Zisserman.
\newblock Very deep convolutional networks for large-scale image recognition.
\newblock {\em ICLR}, 2015.

\bibitem{soudry2014EBP}
D. Soudry, I. Hubara, and R. Meir.
\newblock Expectation backpropagation: Parameter-free training of multilayer
  neural networks with continuous or discrete weights.
\newblock {\em NIPS}, 2018.

\bibitem{veksler1999efficient}
Olga Veksler.
\newblock {\em Efficient graph-based energy minimization methods in computer
  vision}.
\newblock PhD thesis, Cornell University New York, USA, 1999.

\bibitem{wainwright2008graphical}
Martin~J Wainwright, Michael~I Jordan, et~al.
\newblock Graphical models, exponential families, and variational inference.
\newblock {\em Foundations and Trends{\textregistered} in Machine Learning},
  2008.

\bibitem{yin2018binaryrelax}
Penghang Yin, Shuai Zhang, Jiancheng Lyu, Stanley Osher, Yingyong Qi, and Jack
  Xin.
\newblock Binaryrelax: A relaxation approach for training deep neural networks
  with quantized weights.
\newblock {\em SIIMS}, 2018.

\end{thebibliography}
}

\end{document}